\documentclass{article}

\usepackage[final, nonatbib]{neurips_2021}

\usepackage[utf8]{inputenc} 
\usepackage[T1]{fontenc}    
\usepackage{hyperref}       
\usepackage{url}            
\usepackage{booktabs}       
\usepackage{amsfonts}       
\usepackage{nicefrac}       
\usepackage{microtype}      
\usepackage[table,xcdraw]{xcolor}
\usepackage{hhline}
\usepackage{pdfpages}
\usepackage{algorithmicx}
\usepackage[ruled]{algorithm}
\usepackage{algpseudocode}
\usepackage{comment}
\usepackage{afterpage}
\usepackage{amsthm}

\usepackage{amsmath}
\usepackage{graphicx}
\usepackage{mathrsfs}
\usepackage{verbatim}
\usepackage{multirow}
\usepackage{multicol}
\usepackage{tabularx}
\usepackage{siunitx} 
\usepackage{wrapfig}

\newtheorem{theorem}{Theorem}

\newtheorem{lemma}{Lemma}

\DeclareMathOperator*{\argmax}{arg\,max}

\title{Efficient Decompositional Rule Extraction for Deep Neural Networks}

\author{%
  Mateo~Espinosa~Zarlenga\thanks{Department of Computer Science and Technology}\\
  University of Cambridge\\
  Cambridge, UK \\
  \texttt{me466@cam.ac.uk} \\
  \And
  Zohreh~Shams\footnotemark[2]\\
  University of Cambridge\\
  Cambridge, UK \\
  \texttt{zs315@cam.ac.uk} \\
  \And
  Mateja~Jamnik\footnotemark[2]\\
  University of Cambridge\\
  Cambridge, UK \\
  \texttt{mateja.jamnik@cl.cam.ac.uk} \\
}

\begin{document}

\maketitle

\begin{abstract}

In recent years, there has been significant work on increasing both interpretability and debuggability of a Deep Neural Network (DNN) by extracting a rule-based model that approximates its decision boundary.
Nevertheless, current DNN rule extraction methods that consider a DNN's latent space when extracting rules, known as decompositional algorithms, are either restricted to single-layer DNNs or intractable as the size of the DNN or data grows. In this paper, we address these limitations by introducing ECLAIRE, a novel polynomial-time rule extraction algorithm capable of scaling to both large DNN architectures and large training datasets. We evaluate ECLAIRE on a wide variety of tasks, ranging from breast cancer prognosis to particle detection, and show that it consistently extracts more accurate and comprehensible rule sets than the current state-of-the-art methods while using orders of magnitude less computational resources. We make all of our methods available, including a rule set visualisation interface, through the open-source REMIX library (\href{https://github.com/mateoespinosa/remix}{\color{blue}{https://github.com/mateoespinosa/remix}}). 
\end{abstract}

\setcounter{footnote}{0}
\section{Introduction}

As the field of Artificial Intelligence (AI) transitions from traditional symbolic methods into opaque models like Deep Neural Networks (DNNs)~\cite{ai_index_report_2017}, there has been a great amount of concern over the potential consequences of using these methods in safety-critical environments~\cite{xai_survey_trends_and_trajectories, xai_survey_2021}.
The use of such ``black-box'' models in practice can not only lead to uncertainty and delays when debugging and root-causing mistakes in a model's behaviour (as in Uber's infamous accident \cite{self_driving_cars_accident, uber_self_driving_crash}), but it also goes against the long-held conventions of using transparent, interpretable models for life threatening decisions \cite{xai_requirement_american_cancer_criteria, symbolic_meta_learning}. Furthermore, the lack of transparency in DNNs has been shown to lead to user mistrust~\cite{ways_explanations_impact_mental_models}
and to models that misleadingly show promising results for the wrong reasons~\cite{data_leakage_cancer_resolution}.
Such opacity perpetuates a barrier for the practical use of DNNs despite their celebrated success in tasks ranging from computer vision~\cite{resnet}
and language modelling~\cite{bert}
to diagnosis~\cite{retina}.

As a reaction to these concerns, a new line of work in the field of eXplainable AI (XAI) has been developed to explore ways of demystifying a DNN's decision boundary and to facilitate debugging and deployment of these models. A promising proposal in this direction is to make use of the latent space learnt by a DNN to construct a rule-based model (a set or list of IF-THEN rules) that approximates its decision boundary \cite{rule_extraction_survey_2020, rule_extraction_survey_2016}. Such approaches complement a DNN's high performance by offering transparency and debuggability in the form of both ``\textit{local}'' interpretability (i.e., understanding how a single prediction was made) as well as ``\textit{global}'' interpretability (i.e., discovering global patterns in the data that explain how the DNN acts). This is in contrast to existing XAI alternatives, such as feature importance methods \cite{grad_cam, shap}, sample importance methods \cite{sample_based_criticism, sample_based_influence_functions} and counterfactual explanations \cite{counterfactual_wachter, counterfactual_dandl}, where a method only provides either local or global interpretability in their basic form. Moreover, aided by publicly available rule set visualisation tools \cite{viz_rule_bender, viz_rule_matrix}, as well as the possibility of injecting expert hand-crafted rules into rule sets (e.g., as suggested in \cite{rem}), existing work and open-source frameworks for deploying rule-based models can significantly assist model developers, legal moderators, and model users with debugging and deploying a DNN by analysing its approximating rule-based model.

Nevertheless, previous rule extraction methods that consider the inner workings of a DNN to construct their rules, commonly referred to as ``decompositional'' methods \cite{rule_extraction_taxonomy}, tend to have severe limitations. First, with very few exceptions, most decompositional algorithms have been designed to work only on shallow architectures \cite{rule_extraction_survey_2016, rule_extraction_survey_2020}. Second, to the best of our knowledge, all of the approaches that are able to operate on non-trivial DNNs suffer from intractability issues and, therefore, are impractical for tasks that require very deep architectures or have a large training dataset \cite{deepred, rem}. Both of these are in stark contrast with the current direction in DNN architecture design where deep models trained on large datasets dominate over shallow models trained with small datasets \cite{dnn_architecture_survey}.

In this paper we address these limitations by proposing ECLAIRE, a novel decompositional rule extraction algorithm for DNNs. Our method, depicted in Figure~\ref{fig:eclaire_overview}, avoids the intractability common to other decompositional methods by making use of a two-step process that efficiently uses a DNN's latent space to guide its rule extraction. Our main contributions can be summarised as follows:
\begin{itemize}
    \item We introduce ECLAIRE, a novel decompositional rule extraction algorithm which, to the best of our knowledge, is the first polynomial-time decompositional method applicable to arbitrary DNNs.
    
    \item We show that ECLAIRE's rule sets outperform those extracted by current state-of-the-art decompositional methods in both predictive performance and comprehensibility (measured as the number and length of extracted rules). Furthermore, we show that ECLAIRE is able to achieve this while using significantly less computational resources than competing methods. 
    
    \item We open-source our methods, together with a rule visualisation and interaction library, as part of the REMIX package (\href{https://github.com/mateoespinosa/remix}{\color{blue}{https://github.com/mateoespinosa/remix}}).
\end{itemize}

\begin{figure}[!htbp]
    \centering
    \includegraphics[width=0.90\textwidth]{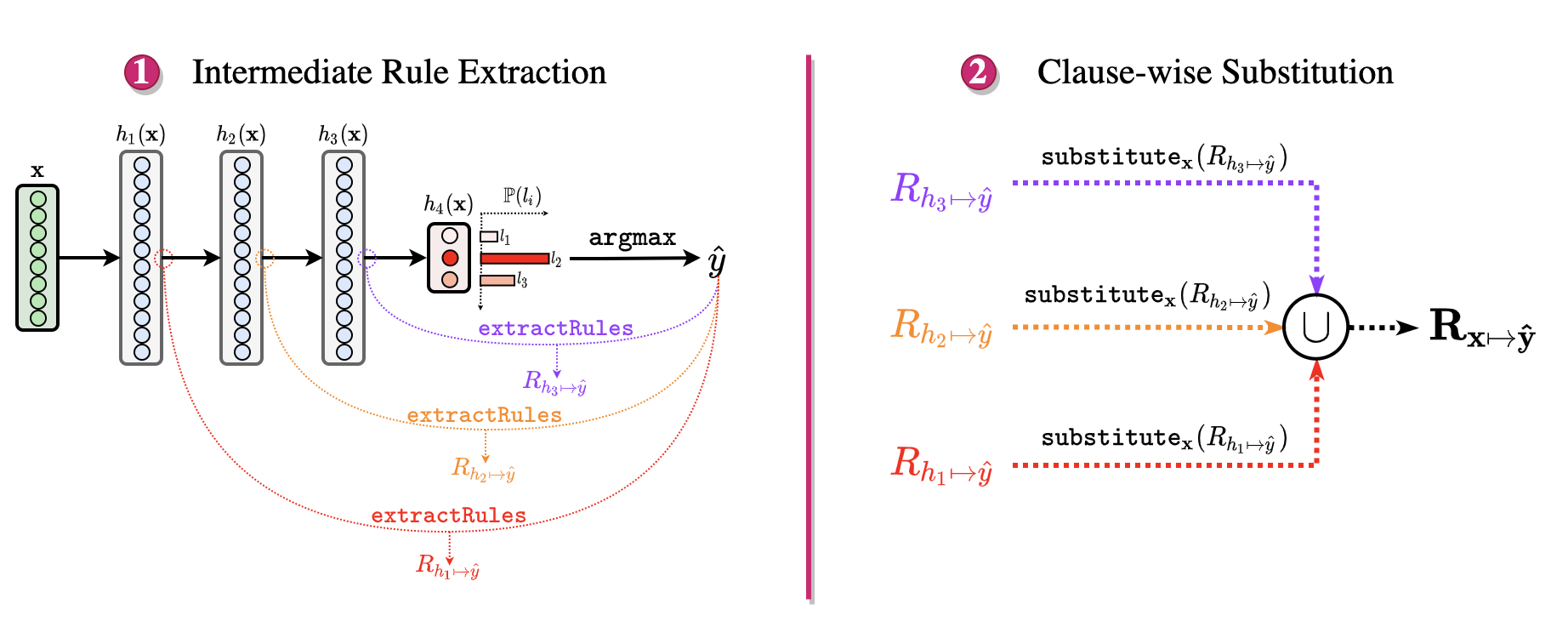}
    \caption{Representation of how ECLAIRE extracts rules from a DNN that outputs a probability distribution over classes $\{l_1, l_2, l_3\}$. ECLAIRE uses a 2-step process in which it (i) extracts rules mapping intermediate (layer-wise) activations to the DNN's predictions and then (ii) performs a clause-wise substitution to produce a rule set $\mathbf{R_{x \mapsto \hat{y}}}$ mapping input activations to output predictions.}
    \label{fig:eclaire_overview}
\end{figure}

\section{Related Work}

Previous work on explaining a DNN's prediction using rule-based explanations can be divided into two lines of work: (i) rule-based model-agnostic local explainability methods and (2) DNN rule extraction methods. While members of the former family of methods, which include algorithms such as LORE~\cite{lore} and Anchors~\cite{anchors}, are able to generate rule-based explanations for the behaviour of a DNN in a point's local neighbourhood, they are unable to provide insights on any \textit{global} patterns captured by the DNN's decision boundary. Hence, given that we are interested in extracting rule-based models that provide both local and global interpretability in a DNN, in this section we focus on DNN rule extraction methods.

The design of DNN rule extraction methods can be divided following a taxonomy proposed by Andrews et al. \cite{rule_extraction_taxonomy} in which a method is classified as being in one of three possible categories: ``\textit{pedagogical}'', ``\textit{decompositional}'', and ``\textit{eclectic}''. \textit{Pedagogical} methods, such as
HYPINV \cite{ped_hypinv},
TREPAN \cite{ped_sampling_dnn},
and Validity Interval Analysis \cite{ped_validity_interval_analysis} extract rules from a DNN by treating it as a black box. A limitation of these methods, however, is that by ignoring a DNN's latent space, they do not take advantage of different hierarchical representations that a DNN was able to learn from its training set. In this work we are concerned with extracting rules from generic multi-layered architectures, and thus in this section we focus on discussing only those decompositional methods that are able to extract rules from arbitrary DNNs. This focus renders the majority of decompositional approaches, such as those in \cite{kt_algorithm, kt_algorithm_generalized, subset_method, rulex_1, rulex_2}, out of scope for our current work due to their restrictive application to only a subset of architectures. Instead, we bring attention to the only two decompositional algorithms applicable to arbitrary DNNs: DeepRED \cite{deepred} and, an optimised version of it, REM-D \cite{rem}.

Borrowing inspiration from CRED \cite{cred}, an earlier single-layer decompositional method, both DeepRED and REM-D use decision trees to induce rule sets that map input features to the truth value of each of the DNN's output classes. The crux of both methods lies in their substitution step in which a set of rules $R_{h_i \mapsto l}$ mapping activations in the $i^\text{th}$ layer to the truth value of class $l$ is rewritten to be a function of activations in the $(i - 1)^\text{th}$ layer. This substitution is performed in a term-wise manner: each term of a rule's premise in $R_{h_i \mapsto l}$ is replaced by a set of new rules which depend on the hidden activations of layer $(i - 1)$. This substitution, however, requires an exponential post-processing step in which rules substituted for a given term need to be distributed in a cartesian-product fashion with rules that were substituted for all terms in the same premise. This results in both REM-D and DeepRED having an exponential asymptotic runtime \cite{rem}.

While algorithmically similar, these methods differ in two implementation-wise optimisations: (1) REM-D uses C5.0 \cite{dt_c5.0} for intermediate decision tree construction rather than C4.5 \cite{dt_c4.5} (a less efficient iteration of C5.0 with typically worse performance) and (2) REM-D reduces memory consumption by substituting intermediate rules as soon as they have been extracted. REM-D's use of C5.0 for intermediate rule extraction, and its early substitution, allow it to empirically outperform DeepRED in terms of its scalability as well as in the size and fidelity of its rule sets \cite{rem_dissertation}.

The last family of rule extraction methods are \textit{eclectic} algorithms, a hybrid between pedagogical and decompositional methods. While eclectic methods such as RX \cite{eclectic_rx}, MofN \cite{eclectic_mofn},
and ERENN\_MHL \cite{eclectic_erenn} have been shown to extract competitive approximating rule sets, these algorithms either assume that the underlying DNN has only one hidden layer or they require full control over the DNN's training pipeline (e.g., they require a specific training procedure to be used for the DNN).

\section{ECLAIRE}
\textbf{Problem Setup} \hspace{2pt} Assume we are given a set of unlabelled training samples $X = \{ \mathbf{x}^{(i)} \in \mathbb{R}^m \}_{i = 1}^N$ and a pre-trained DNN $f_\theta: \mathbb{R}^m \mapsto [0, 1]^L$ such that for all $\mathbf{x} \in X$, $f_\theta(\mathbf{x})$ outputs a probability distribution over labels in set $Y = \{ l_1, l_2, \cdots, l_L \}$. Furthermore, assume $f_\theta$ has $d$ hidden layers and let $h_i(\mathbf{x})$ be the output of the $i$-th layer of $f_\theta$ when fed with input $\mathbf{x}$. In this setup, we let $h_0(\mathbf{x}) = \mathbf{x}$ be the input activation layer and $h_{d + 1}(\mathbf{x}) = f_\theta(\mathbf{x})$ be $f_\theta$'s output probability distribution. Our goal is to use samples in $X$ to construct a rule set $R_{\mathbf{x} \mapsto \hat{y}}$ of IF-THEN rules of the form
\[
    \texttt{IF } \Big( \underbrace{\overbrace{(x_i > v_i)}^\text{term} \wedge (x_j \le v_j) \wedge \cdots \wedge (x_n > v_n)}_{\text{premise/antecedent}} \Big) \texttt{ THEN } \underbrace{l_k}_{\text{conclusion}} 
\]
such that, when subjected to a majority voting\footnote{If a sample does not satisfy any rule premise, then we assign it a default label.} with input $\mathbf{x}$, it accurately predicts the result of evaluating $\argmax_i{f_\theta(\mathbf{x})_i}$. In alignment with existing decompositional methods (e.g., REM-D and DeepRED), in our work we constrain each rule's premise to be a conjunction of terms of the form $(x_i > v_i)$ or $(x_i \le v_i)$, where $x_i$ is the $i$-th input feature of sample $\mathbf{x}$ and $v_i \in \mathbb{R}$ is a learnt threshold.

\textbf{Efficient Clause-wise Rule Extraction} \hspace{2pt}
In order to introduce our method, we begin with the following observation: a neural network $f_\theta$ with $d$ hidden layers provides access to $d$ distinct representations for samples $\mathbf{x} \in X$. This allows us to extend our support training set $X$ by creating $d$ new labelled training sets $\big\{ \{\big( h_i(\mathbf{x}), \argmax_i f_\theta(\mathbf{x})_i \big) \; | \; \mathbf{x} \in X\} \big\}_{i = 1}^d$ from which one can induce $d$ rule sets $R_{h_i \mapsto \hat{y}}$ that map each hidden layer's output to $f_\theta$'s label predictions $\hat{y} = \argmax_i f_\theta(\mathbf{x})_i$. The crux of our method consists of efficiently unifying all $d$ rule sets $\{R_{h_1 \mapsto \hat{y}},
\cdots, R_{h_d \mapsto \hat{y}} \}$ into a single rule set $R_{\mathbf{x} \mapsto \hat{y}}$ that maps input samples $\mathbf{x} \in X$ to $f_\theta$'s predicted labels. In the same way that bagging and feature subsampling are key to the ability of random forests to reduce variance and overfitting \cite{dt_random_forest}, making use of $d$ different representations of the same dataset to construct $R_{\mathbf{x} \mapsto \hat{y}}$ can significantly help reducing the variance and overfitting that would otherwise be present in vanilla rule induction algorithms.

Motivated by these observations, we introduce \textbf{E}fficient \textbf{CLA}use-w\textbf{I}se \textbf{R}ule \textbf{E}xtraction (ECLAIRE). Our method, summarised in Algorithm~\ref{alg:ECLAIRE}, uses an input DNN $f_\theta$ and a set of training samples $X$ to incrementally construct a rule set $R_{\mathbf{x} \mapsto \hat{y}}$ which approximates $f_\theta$ (line~$1$). More specifically, ECLAIRE iterates over each hidden layer $h_i$ and uses a general rule induction algorithm $\psi(\cdot)$
to induce a rule set $R_{h_i \mapsto \hat{y}}$ that maps $h_i$'s activations to the DNN's output predictions
(lines $2$ to $5$). ECLAIRE then proceeds to do a \emph{clause-wise substitution} of rules in $R_{h_i \mapsto \hat{y}}$. This substitution, depicted in Figure~\ref{fig:eclaire_substitution}, iterates over all of the rules in $R_{h_i \mapsto \hat{y}}$ (lines 6-7) and constructs a new temporary rule set $I_{\mathbf{x} \mapsto p}$ mapping input feature activations $\mathbf{x}$ to the truth value of each rule's premise $p$ (lines 8-9). Note that in Algorithm~\ref{alg:ECLAIRE} we use $p(\mathbf{x}^{(i)})$ to represent the result of evaluating premise $p$ (which is a conjunction of terms) on sample $\mathbf{x}^{(i)}$. Finally, we add to $R_{\mathbf{x} \mapsto \hat{y}}$ all of the premises in $I_{\mathbf{x} \mapsto p}$ that have $\texttt{TRUE}$ as a conclusion (lines 10-12). In other words, we add all the rules that approximate premise $p$ as being true and map those rules to the class which $p$ originally predicted. This process is repeated for all intermediate hidden layers and terminates by outputting $R_{\mathbf{x} \mapsto \hat{y}}$ as its result.

\begin{algorithm}[!htbp]
    \small
    \caption{\label{alg:ECLAIRE}ECLAIRE}
    \hspace*{\algorithmicindent} \textbf{Input:} DNN $f_\theta$ with layers $\{h_0, h_1, ..., h_{d+1}\}$ \\
    \hspace*{\algorithmicindent} \textbf{Input:} Training data $X = \{\mathbf{x}^{(1)}, \mathbf{x}^{(2)}, \dots, \mathbf{x}^{(N)}\}$ \\
    \hspace*{\algorithmicindent} \textbf{Hyperparameter:}  Rule extraction algorithm $\psi(\cdot)$ (e.g., rule induction on decision trees learnt via C5.0) \\
    \hspace*{\algorithmicindent} \textbf{Output:} Rule set $R_{\mathbf{x} \mapsto \hat{y}}$
    \begin{algorithmic}[1]
        \State $R_{\mathbf{x} \mapsto \hat{y}} \gets \emptyset$
        
        \State ${\hat{y}^{(1)}}$, ${\hat{y}^{(2)}}$, \dots, ${\hat{y}^{(N)}} \gets \texttt{argmax}\big(h_{d+1}(\mathbf{x}^{(1)})\big), \; \texttt{argmax}\big(h_{d+1}(\mathbf{x}^{(2)})\big), \; \dots, \; \texttt{argmax}\big(h_{d+1}(\mathbf{x}^{(N)})\big)$
        
        \For{hidden layer $i = 1, \dots, d $}
        
            \State ${\mathbf{x^\prime}^{(1)}}$, ${\mathbf{x^\prime}^{(2)}}$, \dots, $\mathbf{x^\prime}^{(N)} \gets h_{i}(\mathbf{x}^{(1)}), \; h_{i}(\mathbf{x}^{(2)}), \; \dots, \; h_{i}(\mathbf{x}^{(N)})$

            \State $R_{h_{i} \mapsto \hat{y}} \gets \psi\big(\{ (\mathbf{x^\prime}^{(1)}, {\hat{y}_1}), (\mathbf{x^\prime}^{(2)}, {\hat{y}_2}), ..., (\mathbf{x^\prime}^{(N)}, {\hat{y}_N})\} \big)$
        
            \For {rule $r \in R_{h_i \mapsto \hat{y}}$}
                \State $p \gets \texttt{getPremise}(r)$
                \State ${\hat{y}_{p}^{(1)}}, \; {\hat{y}_{p}^{(2)}}, \; \dots, \; {\hat{y}_{p}^{(N)}} \gets p\big( \mathbf{x^\prime}^{(1)} \big), p\big( \mathbf{x^\prime}^{(2)} \big), \; \dots, \; p\big( \mathbf{x^\prime}^{(N)} \big)$
                \State $\text{I}_{h_{0} \to p} \gets \psi\big( \{ (\mathbf{x}^{(1)}, {\hat{y}_{p}^{(1)}}), (\mathbf{x}^{(2)}, {\hat{y}_{p}^{(2)}}), ..., (\mathbf{x}^{(N)}, {\hat{y}_{p}^{(N)}}) \} \big)$
                \For {clause $c \in \texttt{getPremisesByConclusion}\big( I_{h_{0} \to p}, \texttt{TRUE} \big)$}
                    \State $R_{\mathbf{x} \mapsto \hat{y}} \gets R_{\mathbf{x} \mapsto \hat{y}} \cup \Big\{ \texttt{IF } c \texttt{ THEN } \texttt{getConclusion}(r) \Big\}$
                \EndFor
            \EndFor
        \EndFor
    \State \textbf{return} $R_{\mathbf{x} \mapsto \hat{y}}$
    \Statex
    \end{algorithmic}
\end{algorithm}

When using a polynomial-time rule extractor $\psi(\cdot)$, ECLAIRE's clause-wise substitution allows it to avoid the expensive and exponential term redistribution that is required in REM-D and DeepRED. This manifests itself in ECLAIRE's runtime which can be shown to be as follows:

\begin{theorem}[ECLAIRE Runtime Complexity]
\label{theorem:eclaire-runtime}
Assume ECLAIRE's intermediate rule extraction algorithm $\psi(\cdot)$ operates by inducing rules from a decision tree that was learnt using a top-down impurity-based algorithm (e.g., C5.0). Furthermore, assume that, when trained with $N$ $m$-dimensional samples, $\psi(\cdot)$'s runtime grows as $\mathcal{O}(N^{p_n} m^{p_m})$, for some $p_n, p_m \in \mathbb{N}$. Given a training set with $N$ samples and a neural network with $d$ hidden layers, such that there are at most $m$ activations in any of its layers, ECLAIRE's runtime will grow as a function of $\mathcal{O}\Big( d N^{\text{max}(3, p_n + 1)} m^{p_m} \Big)$.
\end{theorem}
Theorem~\ref{theorem:eclaire-runtime}, whose proof is in Appendix~\ref{appendix:bound_proof}, highlights a significant computational advantage in ECLAIRE over previous decompositional methods: while both REM-D and DeepRED are known to have an exponential runtime with respect to the training set and the depth of the network \cite{rem},
ECLAIRE exhibits a polynomial-time growth with respect to its input size when using a polynomial-time rule induction method for $\psi(\cdot)$. Furthermore, ECLAIRE's use of a DNN's intermediate representations independently of their topological order in the network implies that, in stark contrast to REM-D and DeepRED, it is agnostic to the DNN's network topology and it can be easily \textit{parallelised}, something we take advantage of in practice by distributing the work in the main loop (lines 3-15) across $n_\text{threads}$ threads. Notice that this is not possible in both REM-D and DeepRED as they extract rules in a sequential manner, processing the DNN's layers in reverse topological order.

\begin{figure}[!htbp]
    \centering
    \includegraphics[width=0.45\textwidth]{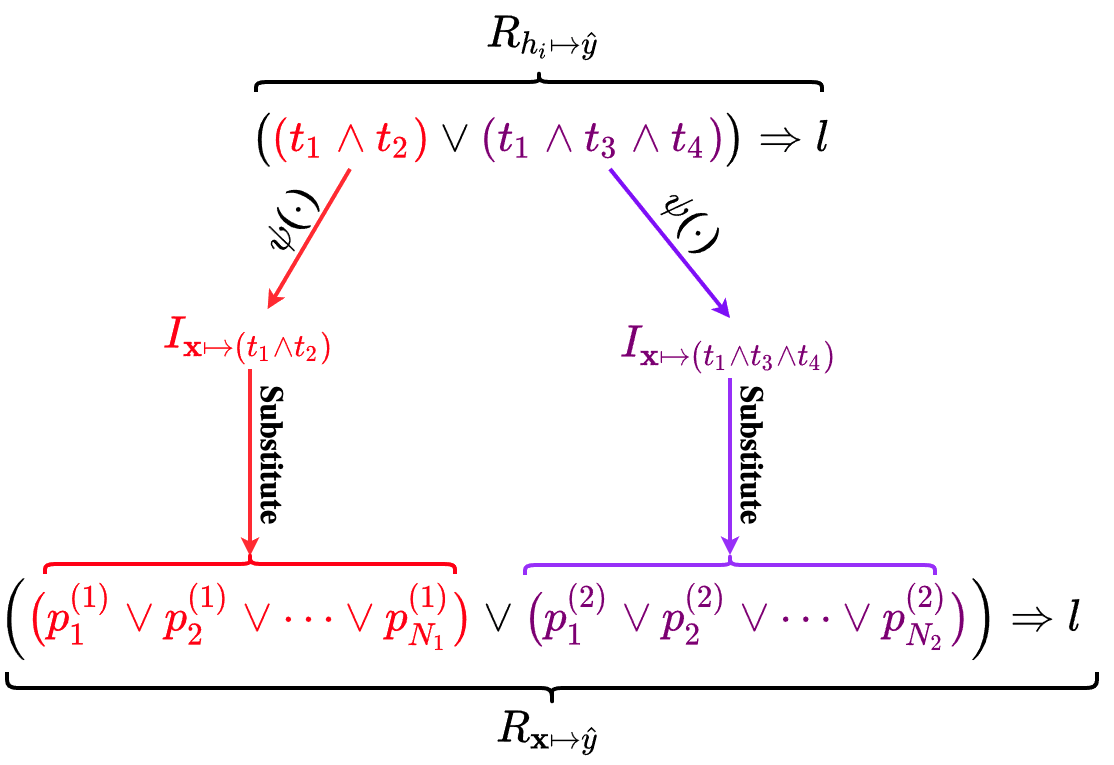}
    \caption{ECLAIRE's substitution step when intermediate rule set $R_{h_i \mapsto \hat{y}}$ has rules $r_1 := (t_1 \wedge t_2) \Rightarrow l$ and $r_2 := (t_1 \wedge t_3 \wedge t_4) \Rightarrow l$. We use $\{t_1, \cdots, t_4\}$ to represent different terms and we let $p_j^{(k)}$ be the premise of the $j$-th rule in $I_{\mathbf{x} \mapsto \texttt{getPremise}(r_k)}$ whose consequence is $\texttt{TRUE}$. In this example we assume that the number of premises with $\texttt{TRUE}$ as a conclusion in $I_{\mathbf{x}\mapsto (t_1 \wedge t_2)}$ and $I_{\mathbf{x}\mapsto (t_1 \wedge t_3 \wedge t_4)}$ is $N_1$ and $N_2$, respectively.
    This results in rule set $R_{\mathbf{x} \mapsto \hat{y}}$ having $N_1 +  N_2$ rules after substitution.
    }
    \label{fig:eclaire_substitution}
\end{figure}

\section{Experiments}
\label{sec:experiments}

\textbf{Setup} \hspace{2pt} For all experiments
we perform a 5-fold stratified cross-validation and, for each metric of interest, we show both the mean and standard deviation errors across all $5$ folds. Furthermore,
all experiments are carried on a $2016$ MacBook Pro with a 2.7 GHz Quad-Core Intel i7 processor and $16$ GB of RAM. Finally, we limit each individual experiment to span at most $n_\text{threads} = 6$ threads when 
using ECLAIRE
and we terminate an experiment if it takes longer than six hours to complete.

\textbf{Datasets and Tasks} \hspace{2pt} We evaluate our methods on six tasks. We begin by using a synthetic task to contrast ECLAIRE's rule sets to those extracted by the current state-of-the-art methods in a controlled setting. We then evaluate ECLAIRE on two breast cancer prognosis tasks based on the METRABRIC \cite{dataset_metabric} dataset to highlight ECLAIRE's applicability to real-world tasks in which transparency is crucial. We proceed to evaluate ECLAIRE's scalability to large datasets and architectures by extracting rules from models trained on two large particle physics datasets (MAGIC \cite{dataset_magic, dataset_uci_ml_repository} and MiniBooNE \cite{dataset_miniboone_experiment, dataset_uci_ml_repository}). Finally, we use the Letter Recognition \cite{dataset_letter, dataset_uci_ml_repository} dataset to explore ECLAIRE's ability to extend to non-binary classification tasks. A summary of each dataset, including its size and label distribution, can be found in Appendix~\ref{appendix:dataset_details}.

\textbf{Baselines and Metrics} \hspace{2pt} In our experiments, we compare ECLAIRE against the only other two decompositional methods that share its applicability: REM-D and DeepRED. Due to a lack of an up-to-date implementation of DeepRED, we follow the work in \cite{rem_dissertation} and use a variation of DeepRED where C5.0 is used instead of C4.5. This results in a variant of DeepRED that outperforms its C4.5 counterpart. For both this implementation of DeepRED, which we refer to as DeepRED\textsuperscript{*}, and the implementation of REM-D, we use as our starting point the code published by Shams et al. \cite{rem, rem_dissertation}. When evaluating ECLAIRE, we use rule sets induced from C5.0 trees as our intermediate rule extractor $\psi(\cdot)$ unless specified otherwise. Although this constraints our evaluation to a single rule extractor (rather than using more recent rule extractors such as Bayesian rule sets~\cite{bayesian_rule_sets} or Column Generation~\cite{column_generation}), this allows us to fairly evaluate the algorithmic design of ECLAIRE against that of DeepRED\textsuperscript{*} and REM-D without strongly biasing our results to our choice of intermediate rule extractor. Nevertheless, notice that from Theorem~\ref{theorem:eclaire-runtime}, and the fact that C5.0's runtime complexity grows as $\mathcal{O}(N^2 m)$ \cite{dt_c5.0}, it follows that the runtime complexity of this instantiation of ECLAIRE grows as $\mathcal{O}(d N^3 m)$. Finally, to highlight the benefits of using a DNN's latent space for rule extraction, we include two non-decompositional baselines: \textit{C5.0} as a baseline end-to-end rule induction method and \textit{PedC5.0} \cite{rem, rem_dissertation} as a baseline pedagogical algorithm. The latter method uses C5.0 to induce a rule set from a training set where each sample is labelled using the DNN's prediction for that sample.

When comparing any two methods, we evaluate the fidelity of the extracted rule sets with respect to the input DNN (defined as the predictive accuracy when sample $\mathbf{x}^{(i)}$ is assigned label $\hat{y}^{(i)} = \argmax_{j}{f_\theta(\mathbf{x}^{(i)})_j}$) and their comprehensibility as measured by their number of rules and their average rule length. Furthermore, as a way to evaluate the tractability of our methods, we keep track of the resources utilised by each method in terms of both wall-clock time and memory. Finally, given that there may be a direct benefit from using the rule sets extracted from the DNN as standalone interpretable models (e.g., as a form of model distillation that generates stronger rule-based models), we include each rule set's accuracy and AUC as part of our evaluation.

\textbf{Model and Hyperparameter Selection} \hspace{2pt} Although ECLAIRE is applicable to arbitrary DNN architectures, in order for us to simplify architecture selection during evaluation, in our experiments we focus on extracting rule sets from multilayer perceptrons (MLPs). Moreover, in order to simulate ECLAIRE's use case, we maintain a strict separation between fine-tuning the DNN we train for a task and fine-tuning the rule extraction algorithm we apply to that model. More specifically, we train several MLP architectures for each task and extract rules only from the best performing model. During rule extraction, we control the growth of intermediate rule sets induced via C5.0 by varying only the minimum number of samples $\mu$ that we require C5.0 to have before splitting a node. Because $\mu$ requires different fine-tuning strategies across different tasks\footnote{This is particularly important in REM-D and DeepRED\textsuperscript{*} as low values of $\mu$ can easily lead to intractability.}, in this section we report only the highest performing rule set we obtain after varying $\mu$ across a task-and-method-specific spectrum. Unless specified otherwise, all other C5.0's hyperparameters, with the exception of \texttt{winnowing} which is always set to true, are left as their default values. We include all details of each task's architecture and hyperparameter selection, including the search spectrum used for $\mu$ for each task, in Appendix~\ref{appendix:experiment_details}.

\subsection{XOR: Evaluating Rule Set Interpretability in Controlled Setup}
\label{sec:experimental-xor}

In this section, we use a synthetic dataset to evaluate the performance of ECLAIRE and our baselines in a controlled environment known to be challenging for vanilla rule induction algorithms and pedagogical rule extraction methods. For this, we make use of a variant of the XOR classification task, in a form commonly used for studying feature selection methods \cite{symbolic_meta_learning, learning_to_explain, kernel_feature_selection}, in which we generate a dataset $\{(\mathbf{x}^{(i)}, y_i)\}_{i = 1}^{1000}$ with 1000 10-dimensional samples such that every data point $\mathbf{x}^{(i)} \in [0, 1]^{10}$ is constructed by independently sampling each dimension from a uniform distribution in $[0, 1]$. In this dataset, we assign a binary label $y_i \in \{0, 1\}$ to each point $\mathbf{x}^{(i)}$ by XOR-ing the result of rounding its first two dimensions as $y_i = \text{round}\big(\mathbf{x}^{(i)}_1\big) \oplus \text{round}\big(\mathbf{x}^{(i)}_2\big)$. 

In order for a classifier to be able to achieve a high performance in this task, it must learn that only the first two dimensions are meaningful for predicting the output label. This property makes this dataset an illustrative example of a task in which decision trees are unable to obtain a high performance without a large training dataset \cite{dt_parity_function, deepred}. In this section, we hypothesise that decompositional methods will outperform both vanilla rule induction algorithms and pedagogical rule extraction methods as their use of a DNN's latent representations during training results in an implicit data augmentation process. Furthermore, we hypothesise that within our decompositional baselines, ECLAIRE will outperform other methods given that its two-step process avoids the need of nesting multiple sequential approximations when merging rule sets as done in both REM-D and DeepRED\textsuperscript{*}. To test both hypothesises, we extract rules from an MLP with hidden layer sizes $\{64, 32, 16\}$ trained on this dataset and report our findings in the XOR rows of Table~\ref{table:all_results}. The MLP's test accuracy is $96.6\% \pm 1.9\%$ and its AUC is $96.6\% \pm 1.9\%$.

\begin{wrapfigure}{R}{0.4\textwidth}
    \centering
    \includegraphics[width=0.39\textwidth]{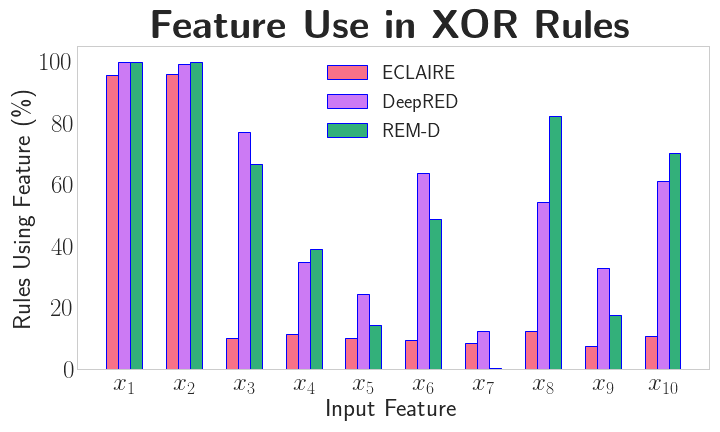}
    \caption{Fraction of rules whose premise thresholds a given feature.}
    \label{fig:xor_feature_importance}
\end{wrapfigure}

As hypothesised, Table~\ref{table:all_results} shows that ECLAIRE outperforms all other decompositional methods in the XOR task in terms of both accuracy and fidelity. We further observe that ECLAIRE uses significantly less resources than REM-D and DeepRED\textsuperscript{*} while generating significantly fewer rules than both. Moreover, note that, as expected, we also observe that both C5.0 and PedC5.0 are unable to capture any meaningful rules and instead generate an empty rule set that acts as a random classifier. Finally, we observe that the standard deviation of ECLAIRE's rule set sizes is relatively small compared to its mean. This is in stark contrast with the observed standard deviations of rule set sizes in both REM-D and DeepRED\textsuperscript{*}, which tend to be even greater than their means. Such behaviour indicates that ECLAIRE is more robust to changes in its training data and input DNN than REM-D and DeepRED\textsuperscript{*}. We conjecture a hypothesis to explain these results in Section~\ref{sec:discussion}.

Our results also indicate that ECLAIRE is better at extracting rules that correctly capture the underlying generating process of the XOR dataset, therefore providing better global insights on what a DNN learnt from its training data. This can be seen when analysing which features were considered important in the rules extracted by ECLAIRE. Such analysis, shown in Figure~\ref{fig:xor_feature_importance}, indicates that rules in ECLAIRE are able to identify the importance that the first two features (i.e., $x_1$ and $x_2$) have on the resulting output class by very rarely conditioning on other features.
This is at odds with what we see in rules extracted by REM-D and DeepRED\textsuperscript{*}, where a significant fraction of rule premises heavily rely on features other than $x_1$ and $x_2$.

\begin{table}
    \centering
    \caption{
        Rule extraction results. For each dataset, we highlight the most competitive value obtained for any metric across all decompositional methods. Furthermore,
        we use a single horizontal line to visually separate decompositional and non-decompositional methods across all datasets. AUC is not included in \textit{Letters} dataset due to its multiclass nature.
    }
    \label{table:all_results}
    \resizebox{0.95\textwidth}{!}{
        \begin{tabular}{c|c|ccccccc}
            & 
            \textbf{Method} &
              \textbf{Accuracy (\%)} &
              \textbf{AUC (\%)} &
              \textbf{Fidelity (\%)} &
              \textbf{Runtime (s)} &
              \textbf{Memory (MB)} &
              \textbf{Rule Set Size} &
              \textbf{Avg Rule Length} \\ \hhline{~|-|-|-|-|-|-|-|-|} \hhline{~|-|-|-|-|-|-|-|-|}

            \multirow{5}{*}{\rotatebox[origin=c]{90}{\textbf{XOR}}}
            & C5.0($\mu = 2$) &
              52.6 $\pm$ 0.02 &
              50 $\pm$ 0 &
              N/A &
              0.10 $\pm$ 0.01 &
              38.35 $\pm$ 13.68 &
              1 $\pm$ 0 &
              0 $\pm$ 0 \\
            & PedC5.0($\mu = 2$) &
              52.6 $\pm$ 0.2 &
              50 $\pm$ 0 &
              53 $\pm$ 0.15 &
              0.19 $\pm$ 0.02 &
              308.44 $\pm$ 27.14 &
              1 $\pm$ 0 &
              0 $\pm$ 0 \\ \hhline{~|-|-|-|-|-|-|-|-|}
            & DeepRed\textsuperscript{*}($\mu = 30$) &
              87.3 $\pm$ 3.7 &
              86.9 $\pm$ 3.9 &
              88.1 $\pm$ 4 &
              390.97 $\pm$ 542.12 &
              20,393.19 $\pm$ 35,889.69 &
              9,149.8 $\pm$ 16,009.6 &
              4.78 $\pm$ 1.77 \\
            & REM-D($\mu = 33$) &
              88.4 $\pm$ 2 &
              88 $\pm$ 2.1 &
              88.2 $\pm$ 3.1 &
              116.12 $\pm$ 205.64 &
              757.98 $\pm$ 14,173.63 &
              3,337 $\pm$ 6,280.65 &
              4.46 $\pm$ 1.55 \\
            & ECLAIRE($\mu = 2$) &
              \textbf{91.8 $\pm$ 2.6} &
              \textbf{91.4 $\pm$ 2.7} &
              \textbf{91.4 $\pm$ 2.4} &
              \textbf{4.15 $\pm$ 0.24} &
              \textbf{523.6 $\pm$ 172.72} &
              \textbf{87 $\pm$ 16.24} &
              \textbf{3.03 $\pm$ 0.23} \\ \hhline{~|-|-|-|-|-|-|-|-|} \hhline{~|-|-|-|-|-|-|-|-|}

        \cellcolor{gray!10} & \multicolumn{1}{c|}{\cellcolor{gray!10}C5.0($\mu = 4$)} &
            \cellcolor{gray!10}91.1 $\pm$ 1.9 &
            \cellcolor{gray!10}87.4 $\pm$ 2.7 &
            \cellcolor{gray!10}N/A &
            \cellcolor{gray!10}20.33 $\pm$ 0.82 &
            \cellcolor{gray!10}303.86 $\pm$ 89.05 &
            \cellcolor{gray!10}17.6 $\pm$ 3.38 &
            \cellcolor{gray!10}2.89 $\pm$ 0.29  \\ 
        \cellcolor{gray!10} & \multicolumn{1}{c|}{\cellcolor{gray!10}PedC5.0($\mu = 2$)} &
            \cellcolor{gray!10}92.7 $\pm$ 0.9 &
            \cellcolor{gray!10}91 $\pm$ 0.6 &
            \cellcolor{gray!10}93 $\pm$ 1.2 &
            \cellcolor{gray!10}20.91 $\pm$ 1.83 &
            \cellcolor{gray!10}1661.06 $\pm$ 2711.3 &
            \cellcolor{gray!10}21.8 $\pm$ 2.99 &
            \cellcolor{gray!10}3.32 $\pm$ 0.2 \\ \hhline{~|-|-|-|-|-|-|-|-|}
        \cellcolor{gray!10} & \multicolumn{1}{c|}{\cellcolor{gray!10}DeepRed\textsuperscript{*}($\mu = 5$)} &
            \cellcolor{gray!10}92 $\pm$ 1.2 &
            \cellcolor{gray!10}89.2 $\pm$ 3 &
            \cellcolor{gray!10}92.3 $\pm$ 2.4 &
            \cellcolor{gray!10}328.67 $\pm$ 173.53 &
            \cellcolor{gray!10}13,262.24 $\pm$ 22,246.47 &
            \cellcolor{gray!10}4,141.4 $\pm$ 8,023.2 &
            \cellcolor{gray!10}5.54 $\pm$ 2.78 \\
        \cellcolor{gray!10} & \multicolumn{1}{c|}{\cellcolor{gray!10}REM-D($\mu = 5$)} &
            \cellcolor{gray!10}91.9 $\pm$ 1.2 &
            \cellcolor{gray!10}89.1 $\pm$ 2.9 &
            \cellcolor{gray!10}92.3 $\pm$ 2.1 &
            \cellcolor{gray!10}154.69 $\pm$ 772.87 &
            \cellcolor{gray!10}\textbf{7,345.99 $\pm$ 7,647.99} &
            \cellcolor{gray!10}1,572.6 $\pm$ 2,935.72 &
            \cellcolor{gray!10}5.36 $\pm$ 2.59 \\
        \multirow{-5}{*}{\rotatebox[origin=c]{90}{\cellcolor{gray!10}\textbf{MB-ER}}} & \multicolumn{1}{c|}{\cellcolor{gray!10}ECLAIRE($\mu = 5$)} &
            \cellcolor{gray!10}\textbf{94.1 $\pm$ 1.6} &
            \cellcolor{gray!10}\textbf{91.8 $\pm$ 2.5} &
            \cellcolor{gray!10}\textbf{94.7 $\pm$ 0.2} &
            \cellcolor{gray!10}\textbf{60.52 $\pm$ 12.1} &
            \cellcolor{gray!10}9,518.72 $\pm$ 5,310.94 &
            \cellcolor{gray!10}\textbf{48.4 $\pm$ 15.28} &
            \cellcolor{gray!10}\textbf{2.84 $\pm$ 0.18} \\ \hhline{~|-|-|-|-|-|-|-|-|} \hhline{~|-|-|-|-|-|-|-|-|} 
        
        \multirow{5}{*}{\rotatebox[origin=c]{90}{\textbf{MB-HIST}}}
        & \multicolumn{1}{c|}{C5.0($\mu = 5$)} &
            89.7 $\pm$ 2 &
            63.5 $\pm$ 8.5 &
            N/A &
            16.06 $\pm$ 0.64 &
            292.21 $\pm$ 67.82 &
            8.2 $\pm$ 5.42 &
            2.14 $\pm$ 1.16  \\
        & \multicolumn{1}{c|}{PedC5.0($\mu = 5$)} &
            87.9 $\pm$ 0.9 &
            72.5 $\pm$ 4.5 &
            89.3 $\pm$ 1 &
            17.16 $\pm$ 0.56 &
            310.79 $\pm$ 83.78 &
            12.8 $\pm$ 3.12 &
            2.75 $\pm$ 0.32 \\ \cline{2-9}
        & \multicolumn{1}{c|}{DeepRed\textsuperscript{*}($\mu = 5$)} &
            88.9 $\pm$ 4.3 &
            73.6 $\pm$ 9.5 &
            89.3 $\pm$ 4.6 &
            523.24 $\pm$ 249.45 &
            \textbf{1,818.16 $\pm$ 881.14} &
            598.8 $\pm$ 355.31 &
            7.75 $\pm$ 2.47 \\
        & \multicolumn{1}{c|}{REM-D($\mu = 5$)} &
            \textbf{89.4 $\pm$ 3.3} &
            74.5 $\pm$ 6.6 &
            89.4 $\pm$ 2.7 &
            128.51 $\pm$ 25.89 &
            11,711.12 $\pm$ 76.2 &
            71.4 $\pm$ 50.07 &
            4.98 $\pm$ 2.04 \\
        & \multicolumn{1}{c|}{ECLAIRE($\mu = 9$)} &
            88.9 $\pm$ 2.3 &
            \textbf{77.4 $\pm$ 3.7} &
            \textbf{89.4 $\pm$ 1.8} &
            \textbf{54.08 $\pm$ 5.68} &
            4,754.45 $\pm$ 4,306.61 &
            \textbf{30 $\pm$ 12.36} &
            \textbf{2.49 $\pm$ 0.21} \\ \hhline{~|-|-|-|-|-|-|-|-|} \hhline{~|-|-|-|-|-|-|-|-|}

        \cellcolor{gray!10} & \multicolumn{1}{c|}{\cellcolor{gray!10}C5.0($\mu = 30$)} &
            \cellcolor{gray!10}81.6 $\pm$ 4.6 &
            \cellcolor{gray!10}75.6 $\pm$ 7.4 &
            \cellcolor{gray!10}N/A & 
            \cellcolor{gray!10}2.48 $\pm$ 0.12 &
            \cellcolor{gray!10}81.54 $\pm$ 66.95 &
            \cellcolor{gray!10}33.2 $\pm$ 1.94 &
            \cellcolor{gray!10}3.21 $\pm$ 0.22 \\
        \cellcolor{gray!10} & \multicolumn{1}{c|}{\cellcolor{gray!10}PedC5.0($\mu = 15$)} &
            \cellcolor{gray!10}82.8 $\pm$ 0.9 &
            \cellcolor{gray!10}77.8 $\pm$ 2.7  &
            \cellcolor{gray!10}85.4 $\pm$ 2.5 &
            \cellcolor{gray!10}2.67 $\pm$ 0.18 &
            \cellcolor{gray!10}889.76 $\pm$ 44.23 &
            \cellcolor{gray!10}57.8 $\pm$ 4.49 &
            \cellcolor{gray!10}3.61 $\pm$ 0.28 \\ \hhline{~|-|-|-|-|-|-|-|-|}
        \cellcolor{gray!10} & \multicolumn{1}{c|}{\cellcolor{gray!10}DeepRed\textsuperscript{*}($\mu = 700$)} &
            \cellcolor{gray!10}78.7 $\pm$ 1 &
            \cellcolor{gray!10}74.9 $\pm$ 3 &
            \cellcolor{gray!10}81 $\pm$ 1.7 &
            \cellcolor{gray!10}708.33 $\pm$ 762.56 &
            \cellcolor{gray!10}11,730.22 $\pm$ 21,684.99 &
            \cellcolor{gray!10}5,142.6 $\pm$ 9,799.13 &
            \cellcolor{gray!10}5.43 $\pm$ 1.91 \\
        \cellcolor{gray!10} & \multicolumn{1}{c|}{\cellcolor{gray!10}REM-D($\mu = 700$)} &
            \cellcolor{gray!10}78.6 $\pm$ 1.1 &
            \cellcolor{gray!10}74.9 $\pm$ 3 &
            \cellcolor{gray!10}81.1 $\pm$ 1.7 &
            \cellcolor{gray!10}355.82 $\pm$ 382.77 &
            \cellcolor{gray!10}8,519.24 $\pm$ 15,301.8 &
            \cellcolor{gray!10}3,616.8 $\pm$ 6,748.41 &
            \cellcolor{gray!10}5.41 $\pm$ 1.88 \\
        \cellcolor{gray!10} \multirow{-5}{*}{\rotatebox[origin=c]{90}{\cellcolor{gray!10}\textbf{MAGIC}}} & \multicolumn{1}{c|}{\cellcolor{gray!10}ECLAIRE($\mu = 50$)} &
            \cellcolor{gray!10}\textbf{84.6 $\pm$ 0.5} &
            \cellcolor{gray!10}\textbf{80.2 $\pm$ 1} &
            \cellcolor{gray!10}\textbf{87.4 $\pm$ 1.2} &
            \cellcolor{gray!10}\textbf{58.26 $\pm$ 10.11} &
            \cellcolor{gray!10}\textbf{1,277.78 $\pm$ 411.14} &
            \cellcolor{gray!10}\textbf{396.2 $\pm$ 74.93} &
            \cellcolor{gray!10}\textbf{3.82 $\pm$ 0.18} \\ \hhline{~|-|-|-|-|-|-|-|-|} \hhline{~|-|-|-|-|-|-|-|-|} 
            
        \multirow{5}{*}{\rotatebox[origin=c]{90}{\textbf{MiniBooNE}}}
        & \multicolumn{1}{c|}{C5.0($\mu = 25$)}
            & 91.7 $\pm$ 0.2 &
            89.3 $\pm$ 0.9 &
            N/A &
            86.57 $\pm$ 4.45 &
            326.19 $\pm$ 65.03 &
            114.2 $\pm$ 12.42 &
            6.04 $\pm$ 0.16  \\
        & \multicolumn{1}{c|}{PedC5.0($\mu = 15$)} &
            91.5 $\pm$ 0.2 &
            90 $\pm$ 0.6 &
            94.1 $\pm$ 0.4 &
            86.76 $\pm$ 2.56 &
            17,608.17 $\pm$ 8,597.74 &
            150.6 $\pm$ 7.74 &
            5.95 $\pm$ 0.26 \\ \cline{2-9}
        & \multicolumn{1}{c|}{DeepRed\textsuperscript{*}($\mu = 0.03N$)} &
            84.2 $\pm$ 3.5 &
            83.5 $\pm$ 1.1 &
            86.4 $\pm$ 2.4 &
            3,983.45 $\pm$ 1,497.53 &
            22,378.28 $\pm$ 620.88 &
            \textbf{262.8 $\pm$ 298.31} &
            \textbf{4.44 $\pm$ 1.49} \\
        & \multicolumn{1}{c|}{REM-D($\mu = 0.025N$)} &
            85.2 $\pm$ 2.3 &
            84.6 $\pm$ 1.4 &
            87.3 $\pm$ 2.1 &
            2,981.38 $\pm$ 1,496.11 &
            19,957.87 $\pm$ 5,879.18 &
            1,182.6 $\pm$ 1,416.93 & 5.23 $\pm$ 1.5 \\
        & \multicolumn{1}{c|}{ECLAIRE($\mu = 0.0008N$)} &
            \textbf{91.4 $\pm$ 0.3} &
            \textbf{90.5 $\pm$ 0.8} &
            \textbf{94.6 $\pm$ 0.3} & \textbf{2,930.79 $\pm$ 160.02} &
            \textbf{3,555.71 $\pm$ 334.75} &
            1,484.8 $\pm$ 131.9 &
            5.81 $\pm$ 0.22 \\ \hhline{~|-|-|-|-|-|-|-|-|} \hhline{~|-|-|-|-|-|-|-|-|}

        \cellcolor{gray!10} & \multicolumn{1}{c|}{\cellcolor{gray!10}C5.0($\mu = 5$)} &
            \cellcolor{gray!10}63.1 $\pm$ 2.5 &
            \cellcolor{gray!10}N/A &
            \cellcolor{gray!10}N/A &
            \cellcolor{gray!10}4.4 $\pm$ 0.42 &
            \cellcolor{gray!10}1,254.97 $\pm$ 39.04 &
            \cellcolor{gray!10}468.6 $\pm$ 6.95 &
            \cellcolor{gray!10}7.51 $\pm$ 0.07  \\
        \cellcolor{gray!10} & \multicolumn{1}{c|}{\cellcolor{gray!10}PedC5.0($\mu = 5$)} &
            \cellcolor{gray!10}60.8 $\pm$ 3.6 &
            \cellcolor{gray!10}N/A &
            \cellcolor{gray!10}60.3 $\pm$ 3.4 &
            \cellcolor{gray!10}4.74 $\pm$ 0.627 &
            \cellcolor{gray!10}1,509.73 $\pm$ 186.85 &
            \cellcolor{gray!10}468 $\pm$ 10.06 &
            \cellcolor{gray!10}7.5 $\pm$ 0.06 \\ \hhline{~|-|-|-|-|-|-|-|-|}
        \cellcolor{gray!10} & \multicolumn{1}{c|}{\cellcolor{gray!10}DeepRed\textsuperscript{*}($\mu = 0.3N$)} &
            \cellcolor{gray!10}4.1 $\pm$ 0.4 &
            \cellcolor{gray!10}N/A &
            \cellcolor{gray!10}4.1 $\pm$ 0.3 &
            \cellcolor{gray!10}2,849.4 $\pm$ 487.31 &
            \cellcolor{gray!10}46,059.52 $\pm$ 34,201.83 &
            \cellcolor{gray!10}16,202.8 $\pm$ 11,208.12 &
            \cellcolor{gray!10}10.06 $\pm$ 0.81 \\
        \cellcolor{gray!10} & \multicolumn{1}{c|}{\cellcolor{gray!10}REM-D($\mu = 0.3N$)} &
            \cellcolor{gray!10}4.2 $\pm$ 0.6 &
            \cellcolor{gray!10}N/A &
            \cellcolor{gray!10}4 $\pm$ 0.4 &
            \cellcolor{gray!10}1,576.74 $\pm$ 248.19 &
            \cellcolor{gray!10}37,917.65 $\pm$ 32,028.71 &
            \cellcolor{gray!10}13,144.2 $\pm$ 10,368.82 &
            \cellcolor{gray!10}10.11 $\pm$ 0.79 \\
        \cellcolor{gray!10} & \multicolumn{1}{c|}{\cellcolor{gray!10}ECLAIRE($\mu = 8$)} &
            \cellcolor{gray!10}55.7 $\pm$ 4.4 &
            \cellcolor{gray!10}N/A &
            \cellcolor{gray!10}55.7 $\pm$ 4.3 &
            \cellcolor{gray!10}\textbf{473.18 $\pm$ 7.83} &
            \cellcolor{gray!10}\textbf{2,802.04 $\pm$ 165.92} &
            \cellcolor{gray!10}\textbf{1,219.4 $\pm$ 70.64} &
            \cellcolor{gray!10}\textbf{5.41 $\pm$ 0.17}  \\
        \cellcolor{gray!10} \multirow{-6}{*}{\rotatebox[origin=c]{90}{\cellcolor{gray!10}\textbf{Letters}}} & \multicolumn{1}{c|}{\cellcolor{gray!10}ECLAIRE\textsuperscript{*}($\mu = 10$)} &
            \cellcolor{gray!10}\textbf{64.9 $\pm$ 2.9} &
            \cellcolor{gray!10}N/A &
            \cellcolor{gray!10}\textbf{64.7 $\pm$ 2.7} &
            \cellcolor{gray!10}487.7 $\pm$ 37.4 &
            \cellcolor{gray!10}4,351.04 $\pm$ 62.05 &
            \cellcolor{gray!10}1,842.4 $\pm$ 109 &
            \cellcolor{gray!10}6.09 $\pm$ 0.07  \\ \hhline{~|-|-|-|-|-|-|-|-|} \hhline{~|-|-|-|-|-|-|-|-|}
        \end{tabular}
    }
\end{table}

\subsection{METABRIC: Medical Applications}
\label{sec:experimental-medical-applications}

In this section we evaluate ECLAIRE in real and safety-critical tasks in which interpretable models are favoured. To achieve this, we evaluate our methods on two medical genomics tasks defined by Shams et al. in \cite{rem} based on the METABRIC dataset. More specifically, we extract rules from models trained to predict immunohistochenical subtypes (MB-ER) and histological tumour subtypes (MB-HIST) from a sequence of mRNA expression profiles. For further details, refer to Appendix~\ref{appendix:dataset_details}.

Using the same DNN architecture and hyperparameters as in \cite{rem}, we train two MLPs with hidden layers $\{128, 16 \}$, one trained on MB-ER and one trained on MB-HIST, and extract rules from these models. In MB-ER, our MLP achieves a test accuracy of $97.9\% \pm 0.3\%$ and test AUC of $95.6\% \pm 1.1\%$ while in MB-HIST our MLP achieves a test accuracy of $91.2\% \pm 3.5\%$ and a test AUC of $88.4\% \pm 3\%$. We summarise these results in the MB-ER and MB-HIST rows of Table~\ref{table:all_results}.

As in XOR, in both of these tasks we observe that ECLAIRE is able to efficiently extract rule sets that are not only significantly smaller than those produced by DeepRED\textsuperscript{*} and REM-D (up to 32x fewer rules on MB-ER) but also higher-performing in their predictive accuracy. Only in MB-HIST we see ECLAIRE underperforming in predictive accuracy when compared to REM-D. However, given the high class imbalance in this dataset (91.3\% of samples are of one class), AUC is a more meaningful metric than accuracy. ECLAIRE's superior AUC in MB-HIST indicates that it is able to extract higher-fidelity rule sets than our baselines in binary tasks with a high class imbalance, a significant advantage for several real-world tasks. Finally, note that although we see similar benefits in ECLAIRE's runtime as in XOR, in these tasks ECLAIRE’s parallelism overhead leads to a higher memory consumption than REM-D and DeepRED\textsuperscript{*}. Nevertheless, although ECALIRE's rule sets are not as small as those produced by PedC5.0 and C5.0, they have fewer than 50 rules, making them easy to inspect and debug for practitioners that may require such a model for prognosis purposes.

\subsection{MiniBooNE and MAGIC: Scalability Study}
\label{sec:experimental-physics-scalability}

In this section we discuss a series of experiments that evaluate how ECLAIRE and our baselines scale when the number of samples and the depth of the architecture increase. For this, we use two particle physics binary classification tasks that offer a large training set: MAGIC \cite{dataset_magic, dataset_uci_ml_repository}, with 19,020 training samples, and MiniBooNE \cite{dataset_miniboone_detector, dataset_uci_ml_repository}, with 130,065 training samples. We train an MLP with hidden layers $\{64, 32, 16 \}$ on MAGIC and an MLP with hidden layers $\{128, 64, 32, 16, 8 \}$ on MiniBooNE and report those results in Table~\ref{table:all_results}. The MLP trained on MAGIC achieves a test accuracy of $86.3\% \pm 1.8\%$ and a test AUC of $86.3\% \pm 1.8\%$ while the MLP trained on MiniBooNE achieves a test accuracy of 97.9\% $\pm$ 0.3\% and a test AUC of 93\% $\pm$ 0.2\%.

In these experiments, we came across significant difficulties when using REM-D and DeepRED\textsuperscript{*}. This was because most of their configurations led to runs that took more than 6 hours in total, either due to intractable rule extraction or intractable inference caused by large extracted rule sets. Thus, we increased the value of $\mu$ to be a significant fraction of $N$, the number of points in the dataset, and we show the best value we obtain by varying $\mu$ as indicated in Appendix~\ref{sec:appendix-miniboone}. Nevertheless, notice that although such amounts of pruning lead to smaller rule sets than ECLAIRE's in MiniBooNE, these are overpruned rule sets which underperform in the test set.

Our experiments show that ECLAIRE can scale to both large architectures and large datasets without the need for excessively large values of $\mu$ (i.e., no need for excessive intermediate rule set pruning). Furthermore, we can see that ECLAIRE outperforms all of the baselines in both tasks in terms of its AUC and fidelity. Only in its accuracy it is marginally outperformed by C5.0 and PedC5.0 in the MiniBooNE task. Nevertheless, given the class imbalance in this task, the benefit from ECLAIRE's AUC improvements out-weights its marginally worse accuracy.

\subsection{Letter Recognition: Multiclass Classification}
\label{sec:experimental-multi-class}

Although both REM-D and DeepRED were developed to be applicable to multiclass tasks, they have never been formally evaluated in such tasks. In this section we fill this gap by extracting rules from a model trained on the Letter Recognition Dataset \cite{dataset_letter, dataset_uci_ml_repository}, a classification task with 26 target labels. We proceed by extracting rules from a trained MLP with hidden layers $\{128, 64 \}$ which achieves a test accuracy of $95\% \pm 0.2\%$. We summarise our results in Table~\ref{table:all_results} and include further experiments using an alternative tree induction algorithm in Appendix~\ref{sec:appendix-letter}.

As in other data-intensive tasks, we first observe a difficultly in tractably extracting rules using \hbox{REM-D} and DeepRED\textsuperscript{*} without large values of $\mu$. Even when we find a configuration for these methods that terminates in the allotted time, we are unable to observe results that are significantly better than random. Second, although ECLAIRE is able to outperform both DeepRED\textsuperscript{*} and REM-D, it struggles to outperform PedC5.0 and C5.0. We believe this to be the case due to possible class imbalance in intermediate rule sets: when performing a clause-wise substitution in intermediate rules, and as the number of classes increases, the number of samples that satisfy the premise of an intermediate rule becomes smaller. This implies that when using C5.0 to substitute an intermediate clause, it has to face a highly imbalanced binary classification task, something known to be problematic for decision trees \cite{dt_class_imbalance}. While we are able to slightly alleviate this in our results through C5.0's support for independent class weights, it remains as future work to explore better mechanisms for handling this imbalance.

An interesting observation is that when we include the input features as one of the representations which ECLAIRE can use to extract intermediate rule sets, it results in a rule set that outperforms C5.0 and PedC5.0. This result, shown at the bottom of Table~\ref{table:all_results} as ECLAIRE\textsuperscript{*}, suggests that using a DNN's input features as one of the representations that ECLAIRE has access to can be beneficial for the performance of its end rule set. This can be thought of as an eclectic variation of ECLAIRE and we leave it as future research to explore whether this improvement generalises across domains.

\section{Discussion}
\label{sec:discussion}

\textbf{Variance Reduction in ECLAIRE} \hspace{2pt} In \cite{rem}, Shams et al. observe a very high variance in the size of rule sets extracted with REM-D and postulate that this is due to different neural network initialisations leading to latent spaces that vary in interpretability properties. Nevertheless, our results in Table~\ref{table:all_results} suggest that this may not be the case, as ECLAIRE's variance is significantly smaller than that observed in both REM-D and DeepRED\textsuperscript{*}. Inspired by results on convergence learning \cite{convergent_learning, network_dissection}, in this paper we challenge this hypothesis by arguing that the variance observed in REM-D, and therefore in DeepRED, is not the result of different initialisations leading to different levels of interpretability in a DNN's latent space, but it is rather the result of a crucial design factor in \mbox{REM-D}: REM-D uses rule sets induced from decision trees to substitute terms in its intermediate rule sets, forcing the need of a cartesian-product redistribution of terms afterwards. Decision trees, being high-variance classifiers \cite{dt_high_variance}, are highly sensitive to their training sets. This implies that when REM-D substitutes the same term in an intermediate rule using two different datasets (as it would happen across two different training folds), it can generate vastly different rule sets. Hence, because during its term-wise substitution REM-D replaces each term in an intermediate rule's premise with a set of rules extracted from a decision tree, the variance in the number of rules after one substitution grows in a multiplicative fashion. This is because all newly substituted terms in a clause need to be redistributed in a cartesian-product (a visual representation of this process can be found in Appendix~\ref{appendix:rem_d_substitution}). Therefore, we argue that the significant reduction in variance we observe in our evaluation of ECLAIRE is due to the fact that ECLAIRE performs a clause-wise substitution rather than a term-wise substitution, hence avoiding a multiplicative growth in the number of rules. This hypothesis explains the results in Table~\ref{table:all_results} and highlights a significant advantage of ECLAIRE over other decompositional methods.

\textbf{Growth Coping Mechanisms} \hspace{2pt} Despite ECLAIRE's superior performance compared to our baselines, Table~\ref{table:all_results} shows that ECLAIRE can still take a considerable amount of time to extract high-performing rule sets in large datasets and architectures. Furthermore, we observe a growth in the number of rules extracted by ECLAIRE as the amount of data and the model size increase. As a way to alleviate this growth, we explore four different coping mechanisms: (1) \textit{intermediate rule pruning}, where, as in \cite{rem}, we increase the value of $\mu$ to generate smaller intermediate rule sets, (2) \textit{hidden representation subsampling}, where we generate intermediate rule sets only from a subset of the DNN's intermediate representations, (3) \textit{training set subsampling}, where, as in \cite{deepred}, only a fraction of the DNN's training set is used for rule extraction, and (4) \textit{intermediate rule ranking}, where the lowest $p\%$ of rules in intermediate rule sets, as ranked by their confidence levels, are dropped. We summarise these experiments in Appendix~\ref{appendix:coping_mechanisms}. 

Our results indicate that the performance of ECLAIRE's rule sets is robust to heavy training data subsampling as well as significant intermediate representation subsampling and intermediate rule pruning via confidence ranking. Specifically, we observe that ECLAIRE is very \textit{data efficient}: one can subsample up to $50\%$ of the training set during rule extraction, causing a $\sim 50\%$ drop in extraction time and rule set size, without observing an absolute loss of more than $0.5\%$ in the resulting rule set fidelity. Similar results are observed when one subsamples every other two hidden representations when generating intermediate rule sets or drops the lowest $25\%$ of rules in intermediate rule sets. Nevertheless, our experiments also highlight a clear limitation in ECLAIRE: for achieving optimal performance, ECLAIRE requires significant task-specific fine tuning of its $\mu$ hyperparameter. While in practice we observe that selecting $\mu$ from $[10^{-5}N, 10^{-4}N]$ yields good results, we leave the exploration of more systematic strategies for selecting $\mu$ to future work.

\section{Conclusion}
With the ubiquitous use of DNNs, it has become increasingly important to be able to explain, as well as debug, the decisions made by production-sized DNNs. In this work we aim to fill this gap by proposing ECLAIRE, a novel decompositional rule extraction algorithm, which, to the best of our knowledge, is the first polynomial-time decompositional algorithm able to scale to both large training sets and deep architectures. In contrast to previous work which focuses on using intermediate representations in a sequential manner, ECLAIRE exploits these representations in parallel to build an ensemble of classifiers that can then be efficiently combined into a single rule set. We evaluate ECLAIRE on six different tasks and show that it consistently outperforms the previous state-of-the-art decompositional methods and vanilla rule induction methods in terms of the accuracy and fidelity of its output rule set. At the same time, ECLAIRE uses orders of magnitude less resources and produces fewer rules than current state-of-the-art decompositional methods. Excited with the prospects that can arise from this work, and to encourage the use of our algorithms in both industry and research, we make all of our methods, together with an extensive set of visualisation and inspection tools, publicly available as part of the REMIX library. Follow up research could explore generalising ECLAIRE to regression tasks, as well as to RNN architectures, and could include work on alleviating the effects of class imbalance in multiclass settings.

\section*{Acknowledgements}

MEZ acknowledges support from the Gates Cambridge Trust via the Gates Cambridge Scholarship.

{
    \small
    \bibliographystyle{unsrt} 
    \bibliography{references} 

\begin{thebibliography}{10}

\bibitem{ai_index_report_2017}
Artificial~Intelligence Index.
\newblock {The Artificial Intelligence Index: 2017 Annual Report}.
\newblock Technical report, Technical report, 2017.

\bibitem{xai_survey_trends_and_trajectories}
Ashraf Abdul, Jo~Vermeulen, Danding Wang, Brian~Y Lim, and Mohan Kankanhalli.
\newblock Trends and trajectories for explainable, accountable and intelligible
  systems: An hci research agenda.
\newblock In {\em Proceedings of the 2018 CHI conference on human factors in
  computing systems}, pages 1--18, 2018.

\bibitem{xai_survey_2021}
Nadia Burkart and Marco~F Huber.
\newblock A survey on the explainability of supervised machine learning.
\newblock {\em Journal of Artificial Intelligence Research}, 70:245--317, 2021.

\bibitem{self_driving_cars_accident}
Joan Claybrook and Shaun Kildare.
\newblock Autonomous vehicles: No driver… no regulation?
\newblock {\em Science}, 361(6397):36--37, 2018.

\bibitem{uber_self_driving_crash}
Aarian Marshall.
\newblock The uber crash won’t be the last shocking self-driving death.
\newblock {\em Transportation, Wired, Available at: https://www. wired.
  com/story/uber-self-driving-crash-explanation-lidar-sensors/, Retrieved},
  6(24):18, 2018.

\bibitem{xai_requirement_american_cancer_criteria}
Michael~W Kattan, Kenneth~R Hess, Mahul~B Amin, Ying Lu, Karl~GM Moons,
  Jeffrey~E Gershenwald, Phyllis~A Gimotty, Justin~H Guinney, Susan Halabi,
  Alexander~J Lazar, et~al.
\newblock American joint committee on cancer acceptance criteria for inclusion
  of risk models for individualized prognosis in the practice of precision
  medicine.
\newblock {\em CA: a cancer journal for clinicians}, 66(5):370--374, 2016.

\bibitem{symbolic_meta_learning}
Ahmed~M Alaa and Mihaela van~der Schaar.
\newblock Demystifying black-box models with symbolic metamodels.
\newblock In {\em NeurIPS}, pages 11301--11311, 2019.

\bibitem{ways_explanations_impact_mental_models}
Todd Kulesza, Simone Stumpf, Margaret Burnett, Sherry Yang, Irwin Kwan, and
  Weng-Keen Wong.
\newblock Too much, too little, or just right? ways explanations impact end
  users' mental models.
\newblock In {\em 2013 IEEE Symposium on visual languages and human centric
  computing}, pages 3--10. IEEE, 2013.

\bibitem{data_leakage_cancer_resolution}
Shachar Kaufman, Saharon Rosset, Claudia Perlich, and Ori Stitelman.
\newblock Leakage in data mining: Formulation, detection, and avoidance.
\newblock {\em ACM Transactions on Knowledge Discovery from Data (TKDD)},
  6(4):1--21, 2012.

\bibitem{resnet}
Kaiming He, Xiangyu Zhang, Shaoqing Ren, and Jian Sun.
\newblock Deep residual learning for image recognition.
\newblock In {\em Proceedings of the IEEE conference on computer vision and
  pattern recognition}, pages 770--778, 2016.

\bibitem{bert}
Jacob Devlin, Ming-Wei Chang, Kenton Lee, and Kristina Toutanova.
\newblock {Bert: Pre-training of deep bidirectional transformers for language
  understanding}.
\newblock {\em arXiv preprint arXiv:1810.04805}, 2018.

\bibitem{retina}
Jeffrey De~Fauw, Joseph~R Ledsam, Bernardino Romera-Paredes, Stanislav Nikolov,
  Nenad Tomasev, Sam Blackwell, Harry Askham, Xavier Glorot, Brendan
  O’Donoghue, Daniel Visentin, et~al.
\newblock Clinically applicable deep learning for diagnosis and referral in
  retinal disease.
\newblock {\em Nature medicine}, 24(9):1342--1350, 2018.

\bibitem{rule_extraction_survey_2020}
Congjie He, Meng Ma, and Ping Wang.
\newblock Extract interpretability-accuracy balanced rules from artificial
  neural networks: A review.
\newblock {\em Neurocomputing}, 387:346--358, 2020.

\bibitem{rule_extraction_survey_2016}
Tameru Hailesilassie.
\newblock Rule extraction algorithm for deep neural networks: A review.
\newblock {\em arXiv preprint arXiv:1610.05267}, 2016.

\bibitem{grad_cam}
Ramprasaath~R Selvaraju, Michael Cogswell, Abhishek Das, Ramakrishna Vedantam,
  Devi Parikh, and Dhruv Batra.
\newblock Grad-cam: Visual explanations from deep networks via gradient-based
  localization.
\newblock In {\em Proceedings of the IEEE international conference on computer
  vision}, pages 618--626, 2017.

\bibitem{shap}
Scott Lundberg and Su-In Lee.
\newblock A unified approach to interpreting model predictions.
\newblock {\em arXiv preprint arXiv:1705.07874}, 2017.

\bibitem{sample_based_criticism}
Been Kim, Oluwasanmi Koyejo, Rajiv Khanna, et~al.
\newblock Examples are not enough, learn to criticize! criticism for
  interpretability.
\newblock In {\em NIPS}, pages 2280--2288, 2016.

\bibitem{sample_based_influence_functions}
Pang~Wei Koh and Percy Liang.
\newblock Understanding black-box predictions via influence functions.
\newblock In Doina Precup and Yee~Whye Teh, editors, {\em Proceedings of the
  34th International Conference on Machine Learning}, volume~70 of {\em
  Proceedings of Machine Learning Research}, pages 1885--1894, International
  Convention Centre, Sydney, Australia, 06--11 Aug 2017. PMLR.

\bibitem{counterfactual_wachter}
Sandra Wachter, Brent Mittelstadt, and Chris Russell.
\newblock Counterfactual explanations without opening the black box: Automated
  decisions and the gdpr.
\newblock {\em Harv. JL \& Tech.}, 31:841, 2017.

\bibitem{counterfactual_dandl}
Susanne Dandl, Christoph Molnar, Martin Binder, and Bernd Bischl.
\newblock Multi-objective counterfactual explanations.
\newblock In {\em International Conference on Parallel Problem Solving from
  Nature}, pages 448--469. Springer, 2020.

\bibitem{viz_rule_bender}
Adam~M Smith, Wen Xu, Yao Sun, James~R Faeder, and G~Elisabeta Marai.
\newblock {RuleBender: integrated modeling, simulation and visualization for
  rule-based intracellular biochemistry}.
\newblock {\em BMC bioinformatics}, 13(8):1--16, 2012.

\bibitem{viz_rule_matrix}
Yao Ming, Huamin Qu, and Enrico Bertini.
\newblock {RuleMatrix: Visualizing and understanding classifiers with rules}.
\newblock {\em IEEE transactions on visualization and computer graphics},
  25(1):342--352, 2018.

\bibitem{rem}
Zohreh Shams, Botty Dimanov, Sumaiyah Kola, Nikola Simidjievski, Helena~Andres
  Terre, Paul Scherer, Urska Matjasec, Jean Abraham, Mateja Jamnik, and Pietro
  Lio.
\newblock {REM: An Integrative Rule Extraction Methodology for Explainable Data
  Analysis in Healthcare}.
\newblock {\em bioRxiv}, 2021.

\bibitem{rule_extraction_taxonomy}
Robert Andrews, Joachim Diederich, and Alan~B Tickle.
\newblock Survey and critique of techniques for extracting rules from trained
  artificial neural networks.
\newblock {\em Knowledge-based systems}, 8(6):373--389, 1995.

\bibitem{deepred}
Jan~Ruben Zilke, Eneldo~Loza Menc{\'\i}a, and Frederik Janssen.
\newblock Deepred--rule extraction from deep neural networks.
\newblock In {\em International Conference on Discovery Science}, pages
  457--473. Springer, 2016.

\bibitem{dnn_architecture_survey}
Weibo Liu, Zidong Wang, Xiaohui Liu, Nianyin Zeng, Yurong Liu, and Fuad~E
  Alsaadi.
\newblock A survey of deep neural network architectures and their applications.
\newblock {\em Neurocomputing}, 234:11--26, 2017.

\bibitem{lore}
Riccardo Guidotti, Anna Monreale, Salvatore Ruggieri, Dino Pedreschi, Franco
  Turini, and Fosca Giannotti.
\newblock Local rule-based explanations of black box decision systems.
\newblock {\em arXiv preprint arXiv:1805.10820}, 2018.

\bibitem{anchors}
Marco~Tulio Ribeiro, Sameer Singh, and Carlos Guestrin.
\newblock Anchors: High-precision model-agnostic explanations.
\newblock In {\em Proceedings of the AAAI conference on artificial
  intelligence}, volume~32, 2018.

\bibitem{ped_hypinv}
Emad~W Saad and Donald~C Wunsch~II.
\newblock Neural network explanation using inversion.
\newblock {\em Neural networks}, 20(1):78--93, 2007.

\bibitem{ped_sampling_dnn}
Mark~W Craven.
\newblock Extracting comprehensible models from trained neural networks.
\newblock Technical report, University of Wisconsin-Madison Department of
  Computer Sciences, 1996.

\bibitem{ped_validity_interval_analysis}
Sebastian Thrun.
\newblock Extracting rules from artificial neural networks with distributed
  representations.
\newblock {\em Advances in neural information processing systems}, pages
  505--512, 1995.

\bibitem{kt_algorithm}
LiMin Fu.
\newblock Rule learning by searching on adapted nets.
\newblock In {\em AAAI}, volume~91, pages 590--595, 1991.

\bibitem{kt_algorithm_generalized}
LiMin Fu.
\newblock Rule generation from neural networks.
\newblock {\em IEEE Transactions on Systems, Man, and Cybernetics},
  24(8):1114--1124, 1994.

\bibitem{subset_method}
Geoffrey~G Towell and Jude~W Shavlik.
\newblock Extracting refined rules from knowledge-based neural networks.
\newblock {\em Machine learning}, 13(1):71--101, 1993.

\bibitem{rulex_1}
Robert Andrews and Shlomo Geva.
\newblock Rule extraction from a constrained error back propagation mlp.
\newblock In {\em Australian Conference on Neural Networks, Brisbane,
  Queensland}, pages 9--12, 1994.

\bibitem{rulex_2}
Robert Andrews.
\newblock Inserting and extracting knowledge from constrained error
  back-propagation networks.
\newblock In {\em Proceedings of the 6th Australian Conference on Neural
  Networks}. NSW, 1995.

\bibitem{cred}
Makoto Sato and Hiroshi Tsukimoto.
\newblock Rule extraction from neural networks via decision tree induction.
\newblock In {\em IJCNN'01. International Joint Conference on Neural Networks.
  Proceedings (Cat. No. 01CH37222)}, volume~3, pages 1870--1875. IEEE, 2001.

\bibitem{dt_c5.0}
J~Ross Quinlan.
\newblock Data mining tools see5 and c5.0.
\newblock {\em https://www.rulequest.com/see5-info.html}, 2020.

\bibitem{dt_c4.5}
J~Ross Quinlan.
\newblock {\em C4. 5: programs for machine learning}.
\newblock Elsevier, 2014.

\bibitem{rem_dissertation}
Sumaiyah Kola.
\newblock Optimising rule extraction for deep neural networks.
\newblock Master's thesis, University of Cambridge, 11 2020.
\newblock Computer Science Tripos - Part II Dissertation.

\bibitem{eclectic_rx}
Eduardo~R Hruschka and Nelson~FF Ebecken.
\newblock Extracting rules from multilayer perceptrons in classification
  problems: A clustering-based approach.
\newblock {\em Neurocomputing}, 70(1-3):384--397, 2006.

\bibitem{eclectic_mofn}
Geoffrey~G Towell and Jude~W Shavlik.
\newblock Extracting refined rules from knowledge-based neural networks.
\newblock {\em Machine learning}, 13(1):71--101, 1993.

\bibitem{eclectic_erenn}
Manomita Chakraborty, Saroj~Kumar Biswas, and Biswajit Purkayastha.
\newblock Rule extraction from neural network trained using deep belief network
  and back propagation.
\newblock {\em Knowledge and Information Systems}, 62:3753--3781, 2020.

\bibitem{dt_random_forest}
Tin~Kam Ho.
\newblock Random decision forests.
\newblock In {\em Proceedings of 3rd international conference on document
  analysis and recognition}, volume~1, pages 278--282. IEEE, 1995.

\bibitem{dataset_metabric}
Bernard Pereira, Suet-Feung Chin, Oscar~M Rueda, Hans-Kristian~Moen Vollan,
  Elena Provenzano, Helen~A Bardwell, Michelle Pugh, Linda Jones, Roslin
  Russell, Stephen-John Sammut, et~al.
\newblock The somatic mutation profiles of 2,433 breast cancers refine their
  genomic and transcriptomic landscapes.
\newblock {\em Nature communications}, 7(1):1--16, 2016.

\bibitem{dataset_magic}
RK~Bock.
\newblock {Major Atmospheric Gamma Imaging Cherenkov Telescope project
  (MAGIC)}.
\newblock {\em URL: https://archive. ics. uci. edu/ml/datasets/magic+ gamma+
  telescope}, 2007.

\bibitem{dataset_uci_ml_repository}
Arthur Asuncion and David Newman.
\newblock {UCI machine learning repository}, 2007.

\bibitem{dataset_miniboone_experiment}
A~Aguilar, LB~Auerbach, RL~Burman, DO~Caldwell, ED~Church, AK~Cochran,
  JB~Donahue, A~Fazely, GT~Garvey, RM~Gunasingha, et~al.
\newblock Evidence for neutrino oscillations from the observation of $\nu$ e
  appearance in a $\nu$ $\mu$ beam.
\newblock {\em Physical Review D}, 64(11):112007, 2001.

\bibitem{dataset_letter}
Peter~W Frey and David~J Slate.
\newblock {Letter recognition using Holland-style adaptive classifiers}.
\newblock {\em Machine learning}, 6(2):161--182, 1991.

\bibitem{bayesian_rule_sets}
Tong Wang, Cynthia Rudin, Finale Doshi-Velez, Yimin Liu, Erica Klampfl, and
  Perry MacNeille.
\newblock A bayesian framework for learning rule sets for interpretable
  classification.
\newblock {\em The Journal of Machine Learning Research}, 18(1):2357--2393,
  2017.

\bibitem{column_generation}
Sanjeeb Dash, Oktay G{\"u}nl{\"u}k, and Dennis Wei.
\newblock Boolean decision rules via column generation.
\newblock {\em arXiv preprint arXiv:1805.09901}, 2018.

\bibitem{learning_to_explain}
Jianbo Chen, Le~Song, Martin Wainwright, and Michael Jordan.
\newblock Learning to explain: An information-theoretic perspective on model
  interpretation.
\newblock In {\em International Conference on Machine Learning}, pages
  883--892. PMLR, 2018.

\bibitem{kernel_feature_selection}
Jianbo Chen, Mitchell Stern, Martin~J Wainwright, and Michael~I Jordan.
\newblock Kernel feature selection via conditional covariance minimization.
\newblock {\em arXiv preprint arXiv:1707.01164}, 2017.

\bibitem{dt_parity_function}
Guy Blanc, Jane Lange, and Li-Yang Tan.
\newblock Top-down induction of decision trees: rigorous guarantees and
  inherent limitations.
\newblock {\em arXiv preprint arXiv:1911.07375}, 2019.

\bibitem{dataset_miniboone_detector}
AA~Aguilar-Arevalo, CE~Anderson, LM~Bartoszek, AO~Bazarko, SJ~Brice, BC~Brown,
  L~Bugel, J~Cao, L~Coney, JM~Conrad, et~al.
\newblock {The MiniBooNE detector}.
\newblock {\em Nuclear instruments and methods in physics research section a:
  accelerators, spectrometers, detectors and associated equipment},
  599(1):28--46, 2009.

\bibitem{dt_class_imbalance}
Philippe Lenca, St{\'e}phane Lallich, Thanh-Nghi Do, and Nguyen-Khang Pham.
\newblock A comparison of different off-centered entropies to deal with class
  imbalance for decision trees.
\newblock In {\em Pacific-Asia Conference on Knowledge Discovery and Data
  Mining}, pages 634--643. Springer, 2008.

\bibitem{convergent_learning}
Yixuan Li, Jason Yosinski, Jeff Clune, Hod Lipson, and John~E Hopcroft.
\newblock Convergent learning: Do different neural networks learn the same
  representations?
\newblock In {\em FE@ NIPS}, pages 196--212, 2015.

\bibitem{network_dissection}
David Bau, Bolei Zhou, Aditya Khosla, Aude Oliva, and Antonio Torralba.
\newblock Network dissection: Quantifying interpretability of deep visual
  representations.
\newblock In {\em Proceedings of the IEEE conference on computer vision and
  pattern recognition}, pages 6541--6549, 2017.

\bibitem{dt_high_variance}
Thomas~G Dietterich and Eun~Bae Kong.
\newblock Machine learning bias, statistical bias, and statistical variance of
  decision tree algorithms.
\newblock Technical report, Citeseer, 1995.

\bibitem{sgd}
Jack Kiefer, Jacob Wolfowitz, et~al.
\newblock Stochastic estimation of the maximum of a regression function.
\newblock {\em The Annals of Mathematical Statistics}, 23(3):462--466, 1952.

\bibitem{optimizer_adam}
Diederik~P Kingma and Jimmy Ba.
\newblock Adam: A method for stochastic optimization.
\newblock {\em arXiv preprint arXiv:1412.6980}, 2014.

\bibitem{letter_dt_results}
Carla~E Brodley and Paul~E Utgoff.
\newblock {\em Multivariate versus univariate decision trees}.
\newblock Citeseer, 1992.

\bibitem{dt_cart}
Wei-Yin Loh.
\newblock Classification and regression trees.
\newblock {\em Wiley interdisciplinary reviews: data mining and knowledge
  discovery}, 1(1):14--23, 2011.

\bibitem{dt_cost_complexity}
Leo Breiman, Jerome~H Friedman, Richard~A Olshen, and Charles~J Stone.
\newblock {\em Classification and regression trees}.
\newblock Routledge, 2017.

\bibitem{dt_ranking_hill_climbing}
Morteza Mashayekhi and Robin Gras.
\newblock Rule extraction from random forest: the rf+ hc methods.
\newblock In {\em Canadian Conference on Artificial Intelligence}, pages
  223--237. Springer, 2015.

\end{thebibliography}
}

\appendix

\section{ECLAIRE's Asymptotic Bounds}
\label{appendix:bound_proof}

\subsection{Notation}

Throughout this section, we let $|R|$ be the number of rules in a rule set $R$, $\mathcal{T}(R)$ be the set of all \textit{unique terms} in rules in $R$, and $\mathcal{T}_\text{max}(R)$ be the set of terms in the \textit{longest rule} in $R$.

\subsection{Lemmas}
Before showing Theorem~\ref{theorem:eclaire-runtime}, we introduce two useful lemmas related to rule induction:

\begin{lemma}
    \label{lemma:dt_num_terms}
    If $R$ is a rule set induced from a decision tree trained with $N$ samples using a recursive top-down impurity-based algorithm (e.g., C5.0), then the following must hold:
    \[
        \big| \mathcal{T}\big( R \big) \big| \le 2(N - 1) \;\;\;\;\; \text{     and     } \;\;\;\;\; \big| \mathcal{T}_\text{max}\big( R \big) \big| \le (N - 1)
    \]
\end{lemma}
\begin{proof}
    First, notice that a tree generated by any top-down tree induction algorithm cannot have more than $N$ leaf nodes. This is because, in the worst-case scenario, any top-down algorithm may reach its base case $N$ times with each leaf having one training sample in it. Therefore, this implies that a tree cannot have more than $(N - 1)$ split nodes in it. If this would not be true, then it is trivial to see that the binary tree would have to have at least $N + 1$ leaf nodes, a contradiction to our previous observation.

    \noindent Furthermore, notice that during rule induction each split node can contribute with at most two unique terms to the rule set (one per branch) and with at most one term to a single rule (one per root-to-leaf path). Because of this, a single rule cannot have more than $(N - 1)$ unique terms in it and the rule set cannot have more than $2(N - 1)$ unique total terms in it. Finally, note that these bounds can be tight. This can be seen in the decision tree shown in Figure~\ref{fig:tree_max_terms} where each split node must be unique as no gains can be obtained from partitioning the training set twice in the same way in any given path.
  
    \begin{figure*}
        \centering
        \includegraphics[width=0.50\textwidth]{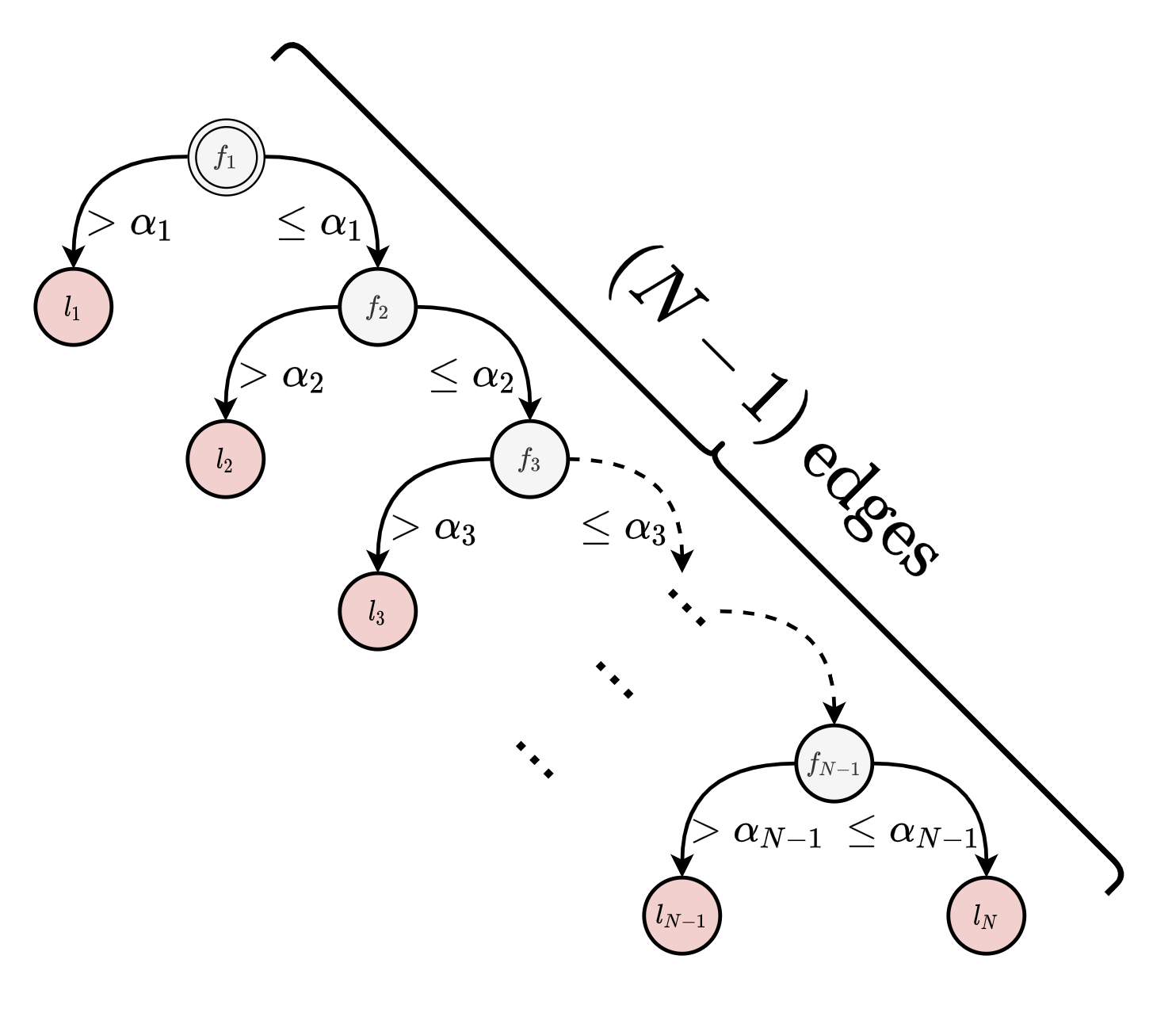}
        \caption{Depiction of an imbalanced decision tree with $N$ leaves and a path with length $(N - 1)$.}
        \label{fig:tree_max_terms}
    \end{figure*}
    
\end{proof}

This lemma can be extended to provide an asymptotic bound on the number of rules induced from a decision tree:
\begin{lemma}
    \label{lemma:dt_num_rules}
    If $R$ is a rule set induced from an impurity-based top-down decision tree trained on $N$ data samples, then it must be the case that $|R| = \mathcal{O}(N)$.
\end{lemma}
\begin{proof}
    An immediate result of our proof of lemma~\ref{lemma:dt_num_terms} is that the number of leaf nodes in a decision tree grows as $\mathcal{O}(N)$. Therefore, because after rule induction there is a one-to-one correspondence between leaves in the decision tree and rules in the output rule set, the number of rules induced from a decision tree must also grow as $\mathcal{O}(N)$.
\end{proof}

Having built an understanding of how rule induction is affected by its training data, we now use these results to analyse ECLAIRE's asymptotic behaviour.

\setcounter{theorem}{0} 
\begin{theorem}[ECLAIRE Runtime Complexity]
Assume ECLAIRE's intermediate rule extraction algorithm $\psi(\cdot)$ operates by inducing rules from a decision tree that was learnt using a top-down impurity-based algorithm (e.g., C5.0). Furthermore, assume that, when trained with $N$ $m$-dimensional samples, $\psi(\cdot)$'s runtime grows as $\mathcal{O}(N^{p_n} m^{p_m})$, for some $p_n, p_m \in \mathbb{N}$. Given a training set with $N$ samples and a neural network with $d$ hidden layers, such that there are at most $m$ activations in any of its layers, ECLAIRE's runtime will grow as a function of $\mathcal{O}\Big( d N^{\text{max}(3, p_n + 1)} m^{p_m} \Big)$.

\end{theorem}

\begin{proof}

    We begin our proof with a high-level summary of ECLAIRE's approach: for each layer $i$ in our network, ECLAIRE extracts an intermediate rule set $R_{h_i \mapsto \hat{y}}$ using rule extractor $\psi(\cdot)$ on intermediate activations $h_i$ and the labels predicted by the input DNN. It then iterates over all premises in $R_{h_i \mapsto \hat{y}}$, extracts a rule set mapping input activations $\mathbf{x}$ to each premise's truth value, and adds the relevant rules into the output set. This procedure implies that the total amount of work done by ECLAIRE to process hidden layer $i$ grows asymptotically as
    \begin{align*}
        \mathcal{O}\Bigg( &(\text{complexity of } \psi(\cdot) \text{ when trained with } N \text{ samples and } m \text{ features}) \\
        &+ (\text{number of rules in } R_{h_i \mapsto \hat{y}}) \times (\text{complexity of evaluating a rule in } R_{h_i \mapsto \hat{y}} \text{ } N \text{ times}) \\
        &+ (\text{number of rules in } R_{h_i \mapsto \hat{y}}) \times (\text{complexity of } \psi(\cdot) \text{ when trained with } N \text{ samples and } m \text{ features}) \\
        &+ (\text{number of rules in } R_{h_i \mapsto \hat{y}}) \times (\text{complexity of post-processing and adding new rules}) \Bigg)
    \end{align*}
    \noindent Which can be rewritten as:
    \begin{align*}
        \mathcal{O}\Bigg( (N^{p_n} m^{p_m}) + \big| R_{h_i \mapsto \hat{y}} \big| \times \Big(
        \big| \mathcal{T}_{max}\big( R_{h_i \mapsto \hat{y}} \big) \big| N
        + (N^{p_n} m^{p_m})
        + \big| \mathcal{T}_\text{max}\big( I^{(t)}_{\mathbf{x} \mapsto p} \big) \big| \big| I^{(r)}_{\mathbf{x} \mapsto p\prime} \big| \Big) \Bigg)
    \end{align*}
    where we define $I^{(t)}_{\mathbf{x} \mapsto p}$ to be the temporary rule set (approximating premise $p$) with the longest rule and we define $I^{(r)}_{\mathbf{x} \mapsto p\prime}$ as the temporary extracted rule set with the largest number of rules in it. Note that we used the fact that, by assumption, $\psi(\cdot)$'s worst case runtime grows as $\mathcal{O}(N^{p_n} m^{p_m})$. Applying lemmas \ref{lemma:dt_num_terms} and \ref{lemma:dt_num_rules} to this expression gives us
    \begin{align*}
        \mathcal{O}\Bigg( (N^{p_n} m^{p_m}) + N \times \Big(
        (N - 1) N
        + (N^{p_n} m^{p_m})
        + (N - 1) N \Big) \Bigg)
    \end{align*}
    \noindent Finally, using the fact that $(N - 1) = \mathcal{O}(N)$ and $ (N^3) = \mathcal{O}(N^{\text{max}(3, p_n + 1)} m^{p_m})$ as both $p_n$ and $p_m$ are natural numbers, and recalling that ECLAIRE performs this much work for each hidden layer, we get that ECLAIRE's total runtime grows as a function of
    \begin{align*}
        \mathcal{O}\Bigg( d\big((N^{p_n} m^{p_m}) + 2(N^3) + (N^{p_n + 1} m^{p_m}) \big) \Bigg) = \mathcal{O}\Bigg( d N^{\text{max}(3, p_n + 1)} m^{p_m} \Bigg)
    \end{align*}
\end{proof}

\section{Dataset Details}
\label{appendix:dataset_details}
In this section we include a brief description of all the non-synthetic datasets used for the tasks described in Section~\ref{sec:experiments}.

\textbf{METABRIC} \cite{dataset_metabric} \hspace{2pt} This dataset consists of a collection of anonymized features extracted from breast cancer tumours in a cohort of $1,980$ patients. It includes clinical traits, gene expression patterns, tumour characteristics, and survival rates for a period of 4 years. The specific tasks we consider in METABRIC, taken from \cite{rem}, are:

\begin{itemize}
    \item $\textbf{Immunohistochemical subtype prediction (MB-ER):}$ for this task we predict immunohistochemical subtypes in breast cancer patients using 1000 mRNA expression patterns. Each sample can be classified as one of two types, ER+ or ER-, which are crucial for determining a patient's treatment.
    \item $\textbf{Histological tumour subtype prediction (MB-HIST):}$ for this task we predict histological subtypes of breast cancer tumours using 1004 mRNA expression profiles. Each sample can be classified as either Invasive Lobular Carcinoma (ILC) or Invasive Ductal Carcinoma (IDC), two most common breast cancer histological subtypes \cite{rem}.
\end{itemize}

\textbf{MAGIC} \cite{dataset_magic, dataset_uci_ml_repository} \hspace{2pt} This is a particle physics dataset in which a signal needs to be classified as being a high-energy gamma ray or some background hadron cosmic radiation. Each of the 19,020 training samples consists of 10 real-valued features that are generated via a Monte Carlo program and a binary label indicating whether the observation corresponds as a high-energy gamma ray (signal) or some background hadron radiation.

\textbf{MiniBooNE} \cite{dataset_miniboone_detector, dataset_uci_ml_repository} \hspace{2pt} This is a particle physics dataset in which one is interested in discriminating electron neutrino events from background events in interactions collected in the MiniBooNE experiment. Each of the 130,065 training samples consists of 50 real-valued features empirically collected in the MiniBooNE experiment and a binary label indicating whether the sample represents a electron neutrino event (signal) or a background event.

\textbf{Letter Recognition} \cite{dataset_letter, dataset_uci_ml_repository} \hspace{2pt} This dataset consists of $20,000$ representations of black-and-white English capital letters labelled with one of 26 classes (A to Z). Each sample is generated by extracting 16 statistical features from the images of each letter.

A summary of the properties of all tasks spanning from these datasets can be found in Table~\ref{table:dataset_summary}. 

\begin{table}[!htbp]
    \caption{Summary of classification tasks used for evaluation.}
    \label{table:dataset_summary}
    \centering
    \resizebox{\textwidth}{!}{
    
        \begin{tabular}{c|ccccc}
            \toprule
            \textbf{Dataset} & \textbf{Samples} & \textbf{Classes} & \textbf{Features} & \textbf{Majority Class (\%)} & \textbf{Domain} \\ \hline
            XOR \cite{symbolic_meta_learning, learning_to_explain} & 1,000   & 2  & 10   & 52.6 & Synthetic               \\
            MB-ER \cite{dataset_metabric, rem}    & 1,980   & 2  & 1,000 & 76   & Healthcare              \\
            MB-HIST \cite{dataset_metabric, rem}  & 1,695   & 2  & 1,004 & 91.3 & Healthcare              \\
            MAGIC \cite{dataset_magic, dataset_uci_ml_repository} & 19,020  & 2  & 10   & 64.8 & Particle Physics        \\
            MiniBooNE \cite{dataset_miniboone_experiment, dataset_uci_ml_repository} & 130,065 & 2  & 50   & 71.9 & Particle Physics        \\
            Letter Recognition \cite{dataset_letter, dataset_uci_ml_repository} & 20,000  & 26 & 16   & 11.7 & Recognition \\ \bottomrule
        \end{tabular}
    }
\end{table}

\section{Model \& Hyperparameter Selection} 
\label{appendix:experiment_details}

\subsection{Model Selection Details}

For each of our tasks, we select a DNN architecture by iterating over several possible architectures and selecting the one with the highest testing performance for rule extraction. For the sake of simplicity, this search process consists of first fixing the number of hidden layers in the network for each task, and then performing a grid-search over different hyperparameters of MLPs with that many hidden layers. In our search, we allow intermediate activations sizes to be chosen from $\{8, 16, 32, 64, 128\}$ and activation functions to be chosen from $\{\texttt{tanh}, \texttt{relu}, \texttt{elu}\}$. Furthermore, to constrain the search space further, we force the number of activations in hidden layer $(i + i)$ to be less than the number of activations in layer $i$. This goes in line with several traditional encoding architectures \cite{dnn_architecture_survey}.

Unless specified otherwise, we train each model to minimise its weighted classification cross-entropy loss for $\text{epochs} \in \{50, 100, 150, 200\}$ using stochastic gradient descent \cite{sgd} (batch size in $\{ 16, 32\}$) and use an Adam optimiser \cite{optimizer_adam} with its default parameters ($\beta_1 = 0.9$, $\beta_2 = 0.999$, $\epsilon = 10^{-7}$, and $\textit{lr}=10^{-3}$) for computing weight updates. To speed this process up, we made use of GPU clusters offered by Google's colab services.

\subsection{Rule Extraction Fine Tuning}
\label{sec:appendix-finetune}

For all of the rule extraction methods we experiment with in our evaluation, we attempt several values for the minimum number of samples per split (i.e., $\mu$) and report only on the rule set that performed the best on the testing dataset. Because this search is very task-specific (due to different dataset sizes and/or architectures), in all of our tasks we define three evaluation parameters for each rule extraction algorithm, namely $\mu_\text{min}$, $\mu_\text{max}$, and $\delta\mu$, and evaluate the performance of rules extracted via that method using values of $\mu \in \{ \mu_\text{min}, (\mu_\text{min} + \delta\mu), (\mu_\text{min} + 2\delta\mu), \cdots, \mu_\text{max}\}$. The values defining the search space for each algorithm in our different tasks are described in the following section.

\subsection{Task-specific Configurations}
In this section we describe the configuration that our grid search produces for each of the tasks we discuss in Section~\ref{sec:experiments}. We also include a description of the search space used for $\mu$ when evaluating different rule extraction methods.

\subsubsection{XOR}
\label{sec:appendix-xor}
Given the simplicity of the XOR task, we constrain the architecture search space to be only over MLPs using 3 hidden layers. This results in the best performing model having hidden layers with sizes \{64, 32, 16\} and $\texttt{tanh}$ activation between them. This model is then trained for 150 epochs using a batch size of 16.

When fine-tuning the different rule extractors, we use $\mu_\text{min} = 2$, $\mu_\text{max} = 15$, and $\delta\mu = 1$ for ECLAIRE, C5.0, and PedC5.0 as they all terminate relatively quickly and obtain good performance without much pruning. REM-D and DeepRED\textsuperscript{*}, however, fail to terminate before their allotted times when using values of $\mu$ below $25$. Because of this, we evaluate them using values of $\mu$ in the range defined by $\mu_\text{min} = 25$, $\mu_\text{max} = 35$, and $\delta\mu = 1$.

\subsubsection{METABRIC}
\label{sec:appendix-metabric}
 For both METABRIC tasks described in Section~\ref{sec:experimental-medical-applications}, we use the exact same architecture and training process used by Shams et al. in \cite{rem}, which was determined through a similar grid search process. This architecture consists of an MLP with hidden layers with sizes $\{128, 16\}$ and \texttt{tanh} activations in between them. For both tasks we train the model for 150 epochs using a batch size of 16.

In our rule extraction fine-tuning, we use $\mu_\text{min} = 2$, $\mu_\text{max} = 15$, and $\delta\mu = 1$ for C5.0, PedC5.0, and ECLAIRE. As suggested by \cite{rem}, and due to their longer run times, we use $\mu_\text{min} = 5$, $\mu_\text{max} = 15$, and $\delta\mu = 5$ for REM-D and DeepRED\textsuperscript{*}.

\subsubsection{MAGIC}
\label{sec:appendix-magic}
In the MAGIC dataset results reported in Section~\ref{sec:experimental-physics-scalability}, we search over architectures with 3 hidden layers and our grid-search results in the best model having layers of size \{64, 32, 16\} with ReLU activations in between them. The best training configuration found is then trained for 200 epochs with a batch size of 32.

In our rule extraction fine-tuning, for ECLAIRE we use $\mu_\text{min} = 50$, $\mu_\text{max} = 200$, and $\delta\mu = 25$ while for both REM-D and DeepRED\textsuperscript{*} we use $\mu_\text{min} = 500$, $\mu_\text{max} = 1000$, and $\delta\mu = 50$. Note that we use large values of $\mu$ for REM-D and DeepRED\textsuperscript{*} as we found it extremely hard to get reasonable extraction times when using less than 500 samples for $\mu$. Finally, for C5.0 and PedC5.0 we search over all $\mu$ values defined by $\mu_\text{min} = 5$, $\mu_\text{max} = 50$, and $\delta\mu = 5$.

\subsubsection{MiniBooNE}
\label{sec:appendix-miniboone}
We found the MiniBooNE task to require the most capacity to obtain good results compared to others experiments in this paper. We force our model architecture to use 5 hidden layers and search for models trained with $\text{epochs} \in \{ 20, 30, 40\}$ given the large training size. This gives us a best performing architecture that uses hidden units \{128, 64, 32, 16, 8\} with an ELU activation in between them which is trained for 30 epochs with a batch size of 16.

In our rule extraction fine-tuning, this task proved to be more complicated than the rest given its training size. For ECLAIRE we use $\mu_\text{min} = 0.0005N$, $\mu_\text{max} = 0.0015N$, and $\delta\mu = 0.0001N$ (where $N$ is the number of training samples) as values below $0.0005N$ result extraction times longer than $6$ hours. For both REM-D and DeepRED\textsuperscript{*}, we increase the minimum value of $\mu$ significantly to get runs that terminate in their allotted times and use $\mu_\text{min} = 0.02N$, $\mu_\text{max} = 0.1N$, and $\delta\mu = 0.005N$. Finally, given their fast extraction times, for C5.0 and PedC5.0 we search over all $\mu$ values defined by $\mu_\text{min} = 5$, $\mu_\text{max} = 50$, and $\delta\mu = 5$.

\subsubsection{Letter Recognition}
\label{sec:appendix-letter}
For the results we report on the Letter Recognition dataset in Section~\ref{sec:experimental-multi-class}, we search over architectures with 2 hidden layers and obtain a best performing model that has layers of size \{128, 64\} with ELU activations in between them. The best training configuration is trained for 150 epochs with a batch size of 32. 

In our rule extraction fine-tuning, we use $\mu_\text{min} = 5$, $\mu_\text{max} = 15$, and $\delta\mu = 1$ for ECLAIRE, C5.0, and PedC5.0. As in the physics datasets, both DeepRED\textsuperscript{*} and REM-D struggle to terminate in a reasonable amount of time unless $\mu$ is significantly high. This is exacerbated by the large number of classes in this task. Because of this, we use $\mu_\text{min} = 0.25N$, $\mu_\text{max} = 0.5N$, and $\delta\mu = 0.05N$ during their fine-tuning process.


\textbf{CART as Intermediate Rule Extractor in Letters} \hspace{0.5 em} While in all of our binary tasks we are able to construct high-performing rule sets with C5.0, we fail to observe this same trend in the Letters dataset, our task of choice for multi-class evaluation. More specifically, Table~\ref{table:all_results} shows that end-to-end C5.0 rule sets are unable to achieve a high performance compared to that previously reported for other induction algorithms \cite{letter_dt_results}. Therefore, in this section we explore the use of CART~\cite{dt_cart} trees for intermediate rule induction and show these results in Table~\ref{table:letter_cart_results}. In these experiments, we compare the performance of rule sets induced from end-to-end CART trees, as well as rule sets induced from CART trees learnt from data that was labelled using the DNN's predictions (which we refer to as \textit{PedCART}), against that of rule sets extracted with ECLAIRE when CART is used as its intermediate rule extractor. Furthermore, for the sake of obtaining a fair comparison between our baselines that is unbiased with respect to the choice of rule extraction algorithm, we also compare ECLAIRE against versions of both REM-D and DeepRED that use CART as an intermediate rule extractor. For clarity, we refer to these versions of our baselines as $\text{ECLAIRE}_\text{CART}$, $\text{REM-D}_\text{CART}$, and $\text{DeepRED}_\text{CART}$, respectively.

In all of these experiments, we control the growth of CART-generated trees using Cost Complexity Pruning (CCP) \cite{dt_cost_complexity} and by varying the number of minimum samples per split $\mu$ as in our previous tasks. For CART, PedCART, and $\text{ECLAIRE}_\text{CART}$ we search over the spectrum defined by $\mu_\text{min} = 0.0001 N$, $\mu_\text{max} = 0.0051 N$, and $\delta\mu = 0.0005 N$. As it was the case when using C5.0 as an intermediate rule extractor, for both $\text{DeepRED}_\text{CART}$ and $\text{REM-D}_\text{CART}$ we require significantly higher values of $\mu$ to terminate in the allotted time. Because of this, we limit our search over $\mu$ to be over the spectrum defined by $\mu_\text{min} = 0.1N$, $\mu_\text{max} = 0.3N$, and $\delta\mu = 0.05N$. This results in overpruned rule sets in $\text{DeepRED}_\text{CART}$ which, although small in size, take orders of magnitude more time to extract than those generated by ECLAIRE.

\begin{table}[!htbp]
    \centering
    \caption{Results of extracting rules from the model used in Section~\ref{sec:experimental-multi-class} in the Letters dataset when using CART as an intermediate rule extractor.}
    \label{table:letter_cart_results}
    \resizebox{\textwidth}{!}{
        \begin{tabular}{ccccccc} 
        \multicolumn{1}{c|}{\textbf{Method}} &
          \textbf{Accuracy (\%)} &
          \textbf{Fidelity (\%)} &
          \textbf{Runtime (s)} &
          \textbf{Memory (MB)} &
          \textbf{Rule set size} &
          \textbf{Avg Rule Length} \\ \hline \hline
        \multicolumn{1}{c|}{CART$(\mu = 0.0001N$)}        & 86 $\pm$ 0.4 & N/A        & 0.89 $\pm$ 0.05       & 3,216.94 $\pm$ 55.05     & 909.2 $\pm$ 18.54       & 11.65 $\pm$ 0.04  \\
        \multicolumn{1}{c|}{PedCART$(\mu = 0.0001N)$}     & 86.3 $\pm$ 0.5 & 85.2 $\pm$ 0.5 & 1.7 $\pm$ 0.1     & 4,122.42 $\pm$ 226.65    & 1,093.6 $\pm$ 19.23        & 11.89 $\pm$ 0.1   \\ \hline
        \multicolumn{1}{c|}{$\text{DeepRED}_\text{CART}(\mu = 0.2 N)$}     & 8.9 $\pm$ 1 & 9 $\pm$ 1 & 3,901.35 $\pm$ 155.47     & \textbf{601.02 $\pm$ 262.75}    & \textbf{66 $\pm$ 6}        & \textbf{3.08 $\pm$ 0.21}   \\
        \multicolumn{1}{c|}{$\text{REM-D}_\text{CART}(\mu = 0.1 N)$}     & 11.1 $\pm$ 2.2 & 11.1 $\pm$ 2.3 & 1,851.07 $\pm$ 1859.9     & 25,992.66 $\pm$ 7,527.34    & 8,835.6 $\pm$ 2,399.13        & 11.08 $\pm$ 0.34   \\
        \multicolumn{1}{c|}{$\text{ECLAIRE}_\text{CART}(\mu = 0.0001 N)$}     & 85.9 $\pm$ 0.5 & 85.3 $\pm$ 0.7 & \textbf{100.18 $\pm$ 8.17}    & 13,629.1 $\pm$ 393.97    & 3,706 $\pm$ 127.47     & 12.34 $\pm$ 0.18  \\
        \multicolumn{1}{c|}{$\text{ECLAIRE}_\text{CART}^*(\mu = 0.0001 N)$}     & \textbf{89.8 $\pm$ 0.4} & \textbf{88.8 $\pm$ 0.4} & 107.92  $\pm$ 16.46    & 16,874.72 $\pm$ 282    & 4,799.6 $\pm$ 120.85     & 12.16 $\pm$ 0.14  \\ \bottomrule
        \end{tabular}
    }
\end{table}

Our results show a significant increase in accuracy when using end-to-end CART trees for rule induction over C5.0 trees (86\% $\pm$ 0.4\% vs 63.1\% $\pm$ 2.5\%). However, we note that this increase comes with a significant rise in the number of rules extracted from CART trees compared to rule sets extracted from C5.0 trees. Regardless, our results also show that the same relative ranking we observe across our baselines in Table~\ref{table:all_results} holds when using CART as intermediate rule extractor: both $\text{ECLAIRE}_\text{CART}$ and $\text{ECLAIRE}_\text{CART}^*$ are able to extract rule sets with higher fidelity than those extracted by PedCART, while $\text{ECLAIRE}_\text{CART}^*$ extracts rule sets that achieve a higher predictive accuracy than those extracted by both end-to-end CART trees and PedCART. Similarly, we observe that $\text{DeepRED}_\text{CART}$ and $\text{REM-D}_\text{CART}$ are unable to extract rule sets that perform better than random.

\section{DeepRED's and REM-D's Substitution Step}
\label{appendix:rem_d_substitution}

\begin{figure}[!htbp]
    \centering
    \includegraphics[width=\textwidth]{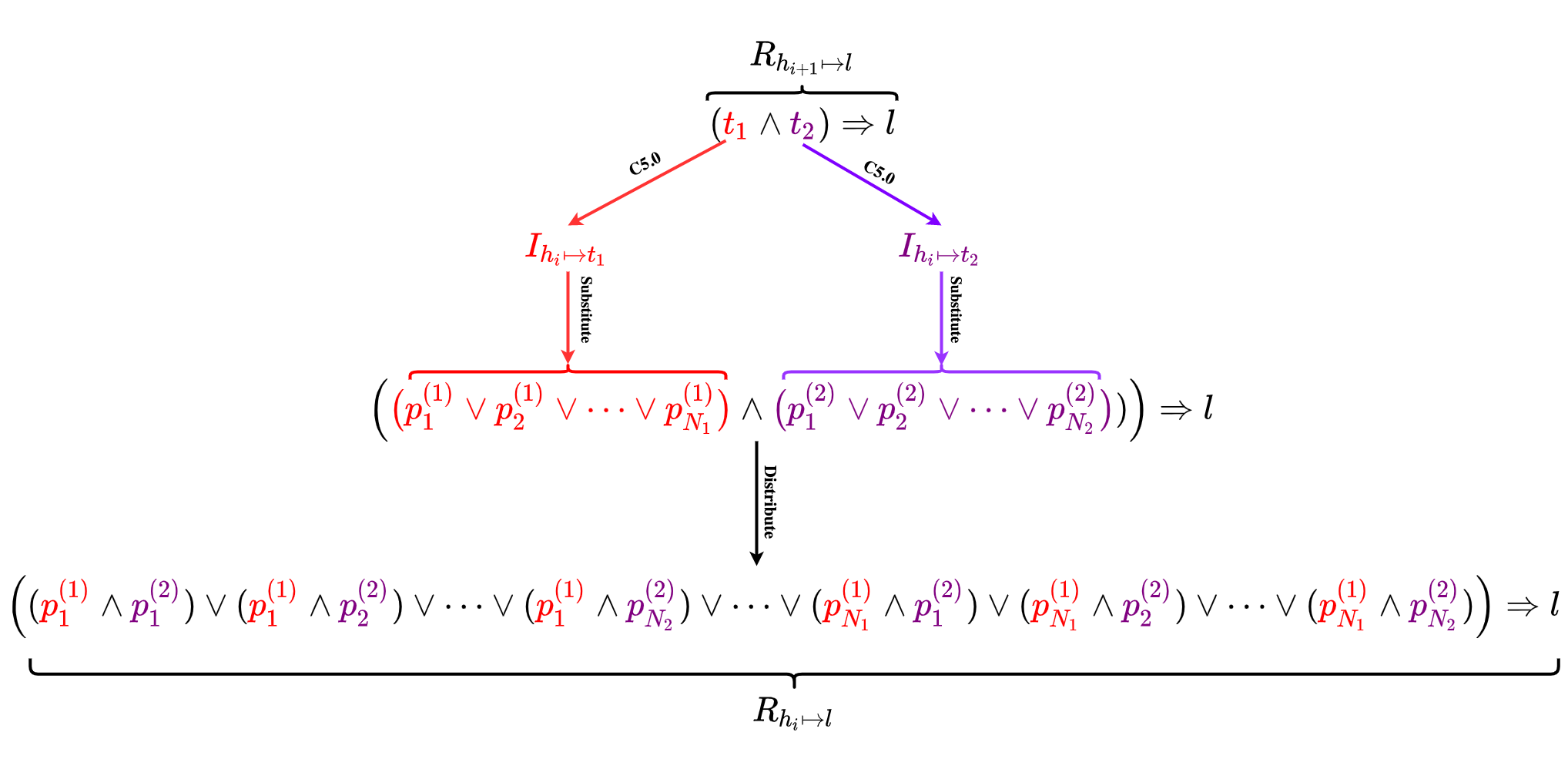}
    \caption{Visual representation of REM-D's and DeepRED's substitution step showing an explosive growth. In this example, REM-D/DeepRED is substituting all terms inside an intermediate rule set $R_{h_{i + 1} \mapsto l}$ which contains a single rule $(t_1 \wedge t_2) \Rightarrow l$ whose premise has two different terms $t_1$ and $t_2$. The algorithm produces temporary rule sets $I_{h_i \mapsto t_k}$ to map activations in layer $h_i$ to the binary truth value of term $t_k$. It then replaces each term $t_k$ in $R_{h_{i + 1} \mapsto l}$ with the set of rules in $I_{h_i \mapsto t_k}$ that have $\texttt{TRUE}$ as a conclusion. Similarly to Figure~\ref{fig:eclaire_substitution}, we let $p_j^{(k)}$ be the premise of the $j$-th rule in $I_{h_i \mapsto t_k}$ whose consequence is $\texttt{TRUE}$. For this example, we assume that the number of premises with $\texttt{TRUE}$ as a conclusion in $I_{h_i \mapsto t_1}$ and $I_{h_i \mapsto t_2}$ is $N_1$ and $N_2$, respectively. This results in rule set $R_{h_{i} \mapsto l}$ having $N_1  N_2$ rules after substitution.
    }
    \label{fig:rem_d_substitution}
\end{figure}

\section{Growth Coping Mechanisms}
\label{appendix:coping_mechanisms}

Although ECLAIRE can scale to large datasets and models, its cubic growth factor implies that it can suffer from slow extraction times. In this section, we include results of experiments in which we explore whether it is possible to alleviate these scaling issues without significantly sacrificing performance. Specifically, we explore four different mechanisms:

\subsubsection*{Intermediate Rule Pruning}

The simplest mechanism for controlling both runtime and comprehensibility is to constrain the size of intermediate rule sets by increasing the value of $\mu$. As seen in Figure~\ref{fig:miniboone_ablation}, varying this parameter can result in a drastic drop in the number of output rules. However, this may come at the cost of fidelity. This behaviour underlines one of the biggest limitations of ECLAIRE: its sensitivity to $\mu$ makes ECLAIRE hard to use in large tasks without the need for extensive fine-tuning. Nevertheless, in most tasks we observe that values of $\mu$ in the $[10^{-5}N, 10^{-4}N]$ range tend to give good performance. Moreover, these experiments also suggest that if one can afford a small drop in performance in favour of a more comprehensible rule set, then increasing $\mu$ to be in the range $[2\times10^{-4}N, 5\times 10^{-4}N]$ can lead to a very comprehensible yet still high-performing rule set.

\begin{figure}[!htbp]
    \centering
    \begin{minipage}{0.325\linewidth}
        \centering
        \centering
        \includegraphics[width=\textwidth]{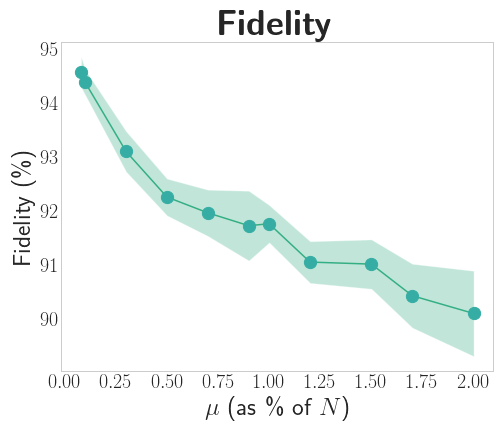}
        \label{fig:miniboone_ablation_fidelity}
    \end{minipage}
    \begin{minipage}{0.325\linewidth}
        \centering
        \includegraphics[width=\textwidth]{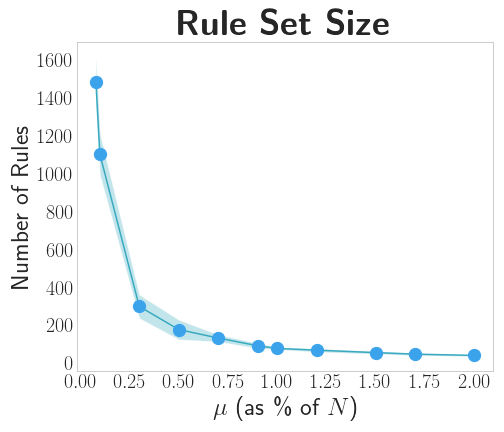}
        \label{fig:miniboone_ablation_ruleset_size}
    \end{minipage}
    \begin{minipage}{0.325\linewidth}
        \centering
        \includegraphics[width=\textwidth]{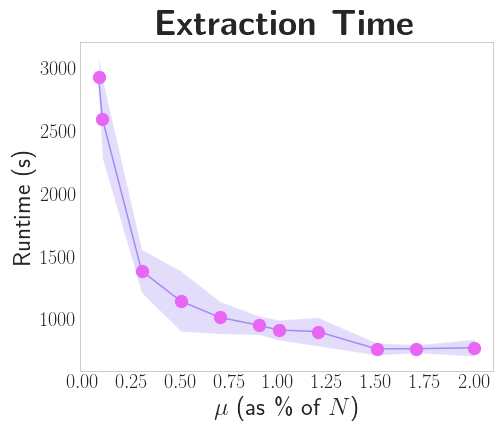}
        \label{fig:miniboone_ablation_runtime}
    \end{minipage}
    \caption{Effects of varying the minimum number of samples required for a split (i.e., $\mu$) in ECLAIRE for MiniBooNE (Section~\ref{sec:experimental-physics-scalability}). These plots, as those in all the figures that follow, are produced by averaging results over 5 folds and show the standard error in shading.
    }
    \label{fig:miniboone_ablation}
\end{figure}

\subsubsection*{Hidden Representation Subsampling}

When dealing with networks with multiple layers in them, one can alleviate ECLAIRE's scalability by extracting intermediate rule sets $R_{h_i \mapsto \hat{y}}$ only for a subset of hidden layers $i \in \mathcal{S} \subset \{1, 2, \dots, d \}$. This forces ECLAIRE to use fewer intermediate representations to build its rule set and results in the algorithm effectively operating on a DNN with $|\mathcal{S}| < d$ hidden layers in it. We explore this possibility in our MiniBooNE network by sampling intermediate layers with different frequencies. Our results, shown in Figure~\ref{fig:miniboone_hidden_sampling}, suggest that one can sample hidden layers at very high frequencies (i.e., every 3 hidden layers) at a very small cost to fidelity; all while halving both the runtime and rule set size. One may therefore benefit from subsampling hidden layers if the input network is deep enough to support this mechanism.

\begin{figure}[!htbp]
    \centering
    \begin{minipage}{0.32\linewidth}
        \centering
        \centering
        \includegraphics[width=\textwidth]{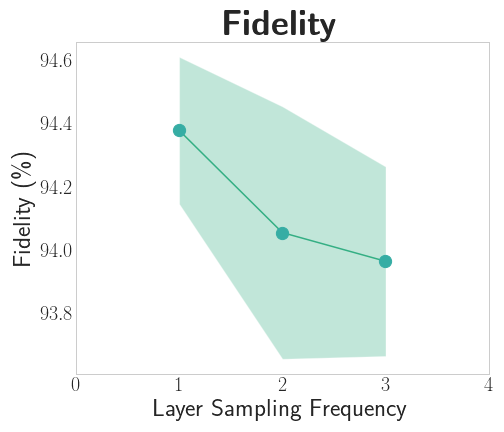}
        \label{fig:miniboone_hidden_sampling_fidelity}
    \end{minipage}
    \begin{minipage}{0.32\linewidth}
        \centering
        \includegraphics[width=\textwidth]{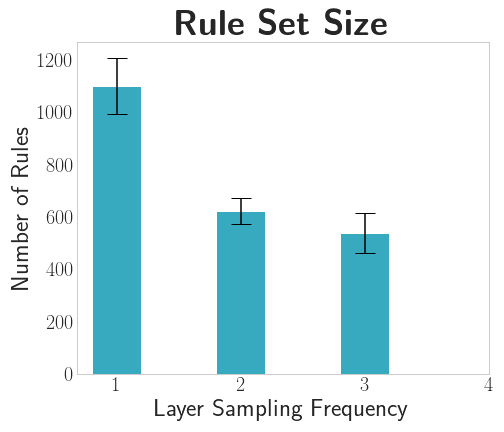}
        \label{fig:miniboone_hidden_sampling_ruleset_size}
    \end{minipage}
    \begin{minipage}{0.32\linewidth}
        \centering
        \includegraphics[width=\textwidth]{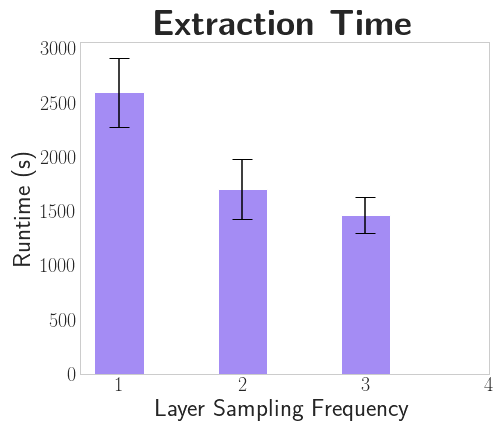}
        \label{fig:miniboone_hidden_sampling_runtime}
    \end{minipage}
    \caption{Effects of sampling hidden layers with a given frequency when constructing ECLAIRE's intermediate rule sets in MiniBooNE (Section~\ref{sec:experimental-physics-scalability}).}
    \label{fig:miniboone_hidden_sampling}
\end{figure}

\subsubsection*{Training Set Subsampling}

If the scalability issue comes from the number of training samples, one can use only a fraction of the available training data for ECLAIRE's rule construction. Our experiments in MiniBooNE show that ECLAIRE is very \textit{data efficient}, maintaining its performance in low-data regimes. These results, shown in Figure~\ref{fig:miniboone_subsampling}, indicate that one can subsample up to 50\% of the training data while incurring a very small drop in fidelity and reducing the extraction time in half. This behaviour is in line with empirical evidence by Zilke et al. suggesting that other decompositional methods can maintain their performance while using a fraction of the DNN's training set \cite{deepred}.

\begin{figure}[!htbp]
    \centering
    \begin{minipage}{0.32\linewidth}
        \centering
        \centering
        \includegraphics[width=\textwidth]{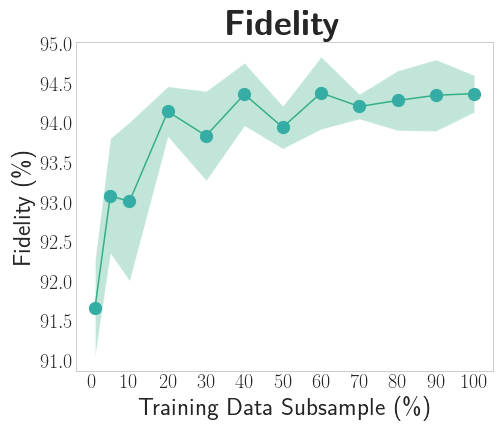}
        \label{fig:miniboone_subsampling_fidelity}
    \end{minipage}
    \begin{minipage}{0.32\linewidth}
        \centering
        \includegraphics[width=\textwidth]{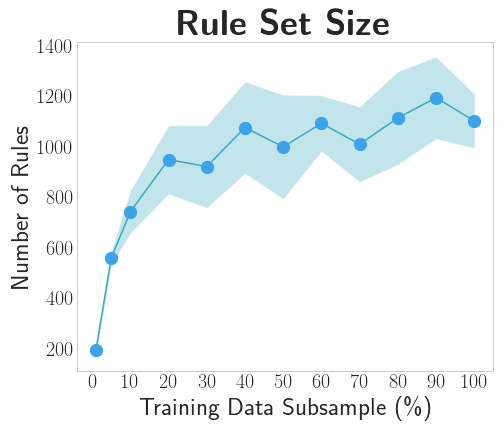}
        \label{fig:miniboone_subsampling_ruleset_size}
    \end{minipage}
    \begin{minipage}{0.32\linewidth}
        \centering
        \includegraphics[width=\textwidth]{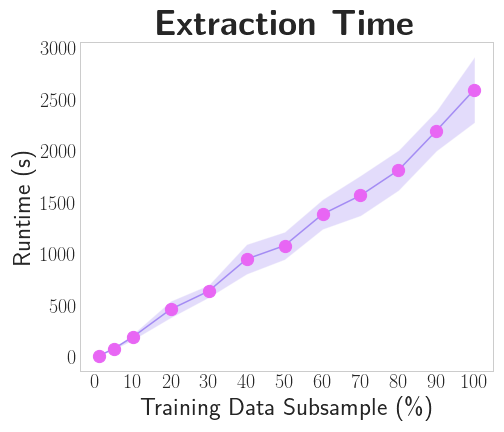}
        \label{fig:miniboone_subsampling_runtime}
    \end{minipage}
    \caption{Effects of subsampling the training set when extracting rules with ECLAIRE in MiniBooNE (Section~\ref{sec:experimental-physics-scalability}).}
    \label{fig:miniboone_subsampling}
\end{figure}

\subsubsection*{Intermediate Rule Ranking}

It has been previously observed that one can drop significant portions of a rule set while incurring only in small performance costs \cite{rem_dissertation, dt_ranking_hill_climbing}. Inspired by this observation, we explore whether ECLAIRE can benefit from dropping the lowest $p\%$ of rules in intermediate rule sets $R_{h_i \mapsto \hat{y}}$ as ranked by their confidence level\footnote{A rule's confidence is defined as the ratio between (a) the number of training samples that satisfy the rule's premise and have the same class as its conclusion and (b) the overall number of training samples that satisfy the rule's premise.}. Our results, shown in Figure~\ref{fig:miniboone_rule_ranking}, suggest that ECLAIRE is robust to substantial pruning of its intermediate rules without loosing much of its performance. For example, one can drop about 25\% of all intermediate rules while experiencing a very small cost in fidelity; all while using less resources and generating fewer rules. Both of these results suggest that dropping a small fraction of rules can help scaling ECLAIRE to large tasks.


\begin{figure}[!htbp]
    \centering
    \begin{minipage}{0.32\linewidth}
        \centering
        \centering
        \includegraphics[width=\textwidth]{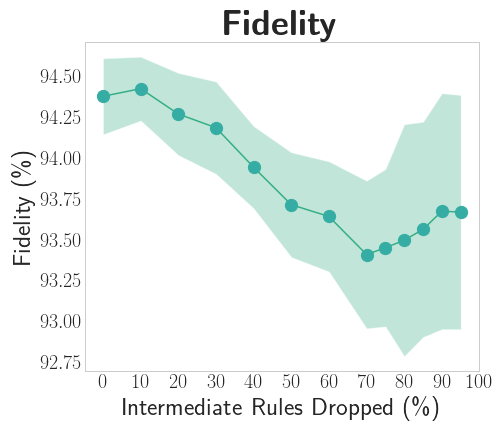}
        \label{fig:miniboone_rule_ranking_fidelity}
    \end{minipage}
    \begin{minipage}{0.32\linewidth}
        \centering
        \includegraphics[width=\textwidth]{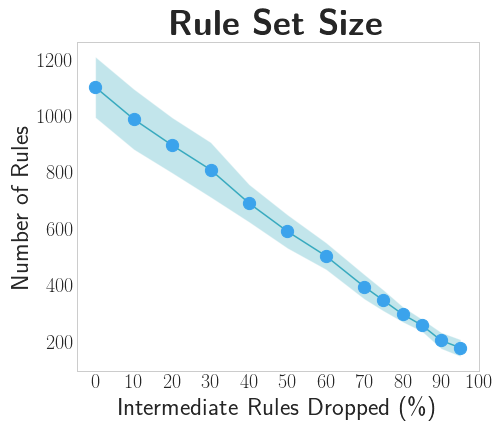}
        \label{fig:miniboone_rule_ranking_ruleset_size}
    \end{minipage}
    \begin{minipage}{0.32\linewidth}
        \centering
        \includegraphics[width=\textwidth]{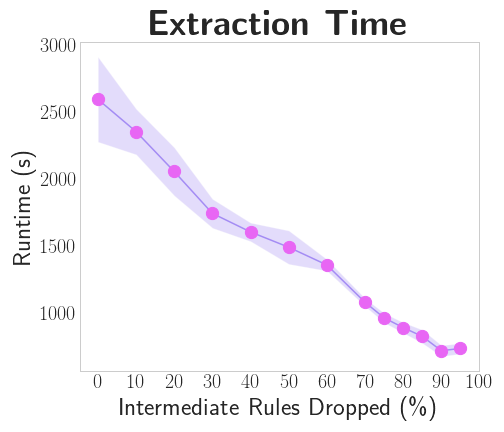}
        \label{fig:miniboone_rule_ranking_runtime}
    \end{minipage}
    \caption{Effects of dropping rules, ranked by their confidences levels, in intermediate rule sets generated by ECLAIRE in MiniBooNE (Section~\ref{sec:experimental-physics-scalability}).}
    \label{fig:miniboone_rule_ranking}
\end{figure}

\end{document}